\documentclass{article}

\usepackage[preprint]{neurips_2025}

\usepackage[utf8]{inputenc} 
\usepackage[T1]{fontenc}    
\usepackage{hyperref}       
\usepackage{url}            
\usepackage{booktabs}       
\usepackage{amsfonts}       
\usepackage{nicefrac}       
\usepackage{microtype}      
\usepackage{xcolor}         

\usepackage{lipsum}
\usepackage{graphicx}
\usepackage{wrapfig}
\usepackage{algorithm}
\usepackage{algpseudocode}
\usepackage{thmtools,thm-restate}
\usepackage{amsthm, amssymb}
\usepackage{colortbl}

\usepackage{amsmath,amsfonts,bm}

\def\eqref#1{equation~\ref{#1}}

\def\1{\bm{1}}

\def\rmI{{\mathbf{I}}}

\def\va{{\bm{a}}}

\def\vc{{\bm{c}}}

\def\vg{{\bm{g}}}

\def\vp{{\bm{p}}}
\def\vq{{\bm{q}}}

\def\vu{{\bm{u}}}
\def\vv{{\bm{v}}}

\def\vx{{\bm{x}}}

\def\vz{{\bm{z}}}

\DeclareMathAlphabet{\mathsfit}{\encodingdefault}{\sfdefault}{m}{sl}
\SetMathAlphabet{\mathsfit}{bold}{\encodingdefault}{\sfdefault}{bx}{n}

\newcommand{\ub}{{\boldsymbol u}}

\newcommand{\pb}{{\boldsymbol p}}

\newcommand{\x}{{\boldsymbol x}}

\newcommand{\m}{{\boldsymbol m}}

\newcommand{\epsilonb}{{\boldsymbol \epsilon}}

\usepackage{capt-of}
\usepackage{diagbox}

\definecolor{C0}{rgb}{0.121569, 0.466667, 0.705882}
\definecolor{C1}{rgb}{1.000000, 0.498039, 0.054902}
\definecolor{C2}{rgb}{0.172549, 0.627451, 0.172549}
\definecolor{C3}{rgb}{0.839216, 0.152941, 0.156863}
\definecolor{C4}{rgb}{0.580392, 0.403922, 0.741176}
\definecolor{C5}{rgb}{0.549020, 0.337255, 0.294118}
\definecolor{C6}{rgb}{0.890196, 0.466667, 0.760784}
\definecolor{C7}{rgb}{0.498039, 0.498039, 0.498039}
\definecolor{C8}{rgb}{0.737255, 0.741176, 0.133333}
\definecolor{C9}{rgb}{0.090196, 0.745098, 0.811765}
\definecolor{trolleygrey}{rgb}{0.5, 0.5, 0.5}

\definecolor{BrickRed}{rgb}{0.6,0,0}
\definecolor{RoyalBlue}{rgb}{0,0,0.8}
\definecolor{Tdgreen}{rgb}{0,0.4,0.7}
\definecolor{pinegreen}{rgb}{0.0, 0.47, 0.44}
\definecolor{cornellred}{rgb}{0.7, 0.11, 0.11}
\definecolor{cadmiumgreen}{rgb}{0.0, 0.42, 0.24}
\definecolor{spirodiscoball}{rgb}{0.06, 0.75, 0.99}
\definecolor{mylightblue}{rgb}{0.85, 0.90, 0.94}
\definecolor{maroon}{cmyk}{0,0.87,0.68,0.32}

\usepackage{pifont}
\definecolor{cfg}{rgb}{0.906, 0.435, 0.318}
\definecolor{cfgpp}{rgb}{0.165, 0.616, 0.561}
\definecolor{cfgnull}{rgb}{0.208, 0.565, 0.953}

\newtheorem{lemma}{Lemma}
\def\eqref#1{Eq.~(\ref{#1})}

\usepackage{animate}
\usepackage{tabularx}

\usepackage{color}
\usepackage{colortbl}
\definecolor{tabfirst}{rgb}{1, 0.7, 0.7} 
\definecolor{tabsecond}{rgb}{1, 0.85, 0.7} 
\definecolor{tabthird}{rgb}{1, 1, 0.7} 

\title{
FlowAlign: Trajectory-Regularized,  Inversion-Free Flow-based Image Editing
}

\author{%
  Jeongsol Kim*, Yeobin Hong*, Jonghyun Park, Jong Chul Ye\\
  KAIST\\
  \texttt{\{jeongsol, yeobin34, jhpark99, jong.ye\}@kaist.ac.kr}\\
 \text{* Equal contribution}\\
}

\begin{document}
\maketitle

\begin{figure}[!h]
    \centering
    \includegraphics[width=\linewidth]{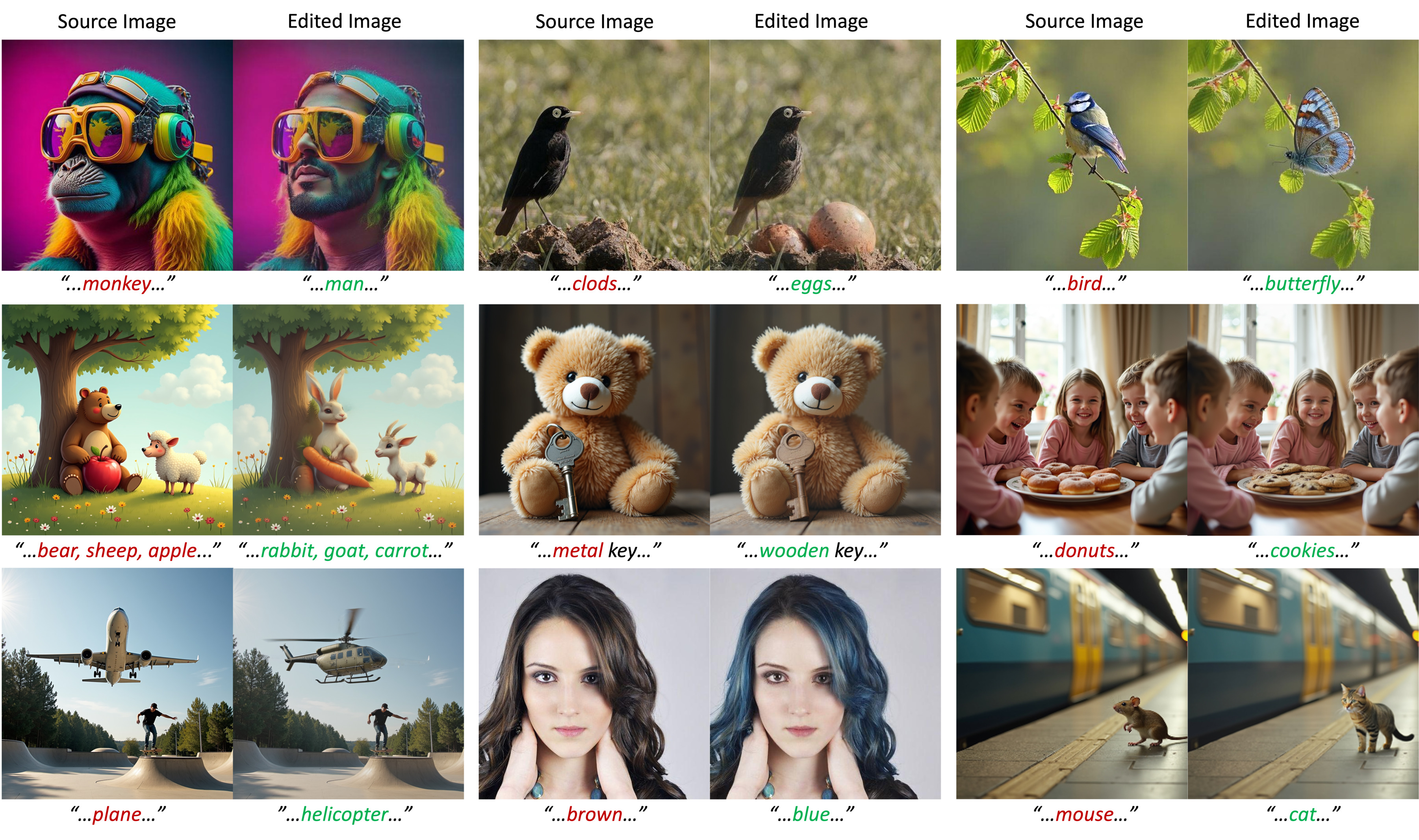}
    \vspace{-0.7cm}
    \caption{Representative editing results produced by FlowAlign, where the red portion of the prompt is replaced with the green portion. Samples are drawn from EditBench and PIEBench. }
    \label{fig:represent}
\end{figure}

\begin{abstract}
Recent inversion-free, flow-based image editing methods such as FlowEdit leverages a pre-trained noise-to-image flow model such as Stable Diffusion 3, 
enabling text-driven manipulation by solving an ordinary differential equation (ODE). While the lack of exact latent inversion is a core advantage of these methods, it often results in unstable editing trajectories and poor source consistency.
To address this limitation, we propose {\em FlowAlign}, a novel inversion-free flow-based framework for consistent image editing with optimal
control-based trajectory control. Specifically, FlowAlign introduces 
source similarity at the terminal point as a regularization term to promote smoother and more consistent trajectories during the editing process.
Notably, our terminal point regularization is shown to  explicitly balance semantic alignment with the edit prompt and structural consistency with the source image along the trajectory. Furthermore, FlowAlign naturally supports reverse editing by simply reversing the ODE trajectory, highliting the reversible and consistent nature of the transformation. Extensive experiments demonstrate that FlowAlign outperforms existing methods in both source preservation and editing controllability.
\end{abstract}

\section{Introduction}
In text-based image editing \cite{meng2021sdedit, mokady2023null, tumanyan2023plug, hertz2023delta, kim2024dreamsampler, kulikov2024flowedit, brooks2023instructpix2pix}, the goal is to transform a source image into a target  based on either textual descriptions of the images or specific editing instructions.
From a distributional perspective, the image editing task can be interpreted as a continuous normalizing flow (CNF)  \cite{papamakarios2021normalizing} that pushes forward a source distribution to a target distribution.
Specifically, we parameterize a velocity field, that uniquely determines the flow, using a neural network. Then, the generative process corresponds to solving an ordinary differential equation (ODE) governed by the trained velocity field.
To reduce the computational cost of simulating ODEs during training for likelihood evaluation, flow matching has been proposed \cite{lipman2023flow}. Its conditional variant enables direct supervision by computing the target velocity, allowing efficient training of flow models between arbitrary distributions.
These flow models includes score-based diffusion models \cite{song2020score,ho2020denoising, song2020denoising} as well as rectified flow models \cite{liu2023flow, esser2024scaling}.

Recently, within the framework of flow-based models, text-to-image generation has achieved significant advancements through improved time discretization, loss weighting, and model architecture, notably based on DiT \cite{peebles2023scalable}. While these foundational models are trained to map samples from a normal distribution to a clean data distribution, additional training—either from scratch or via fine-tuning—is necessary to establish a flow between two arbitrary distributions.

To mitigate this additional cost for constructing a new flow between image distributions, several approaches such as SDEdit~\cite{meng2021sdedit} and Dual Diffusion Implicit Bridge (DDIB)~\cite{su2022dual} have leveraged pre-trained noise-to-image diffusion models.
However, SDEdit requires careful selection of an appropriate initial noise level for editing, while DDIB relies on an inversion process that is prone to errors arising from discretization and an approximated velocity field for subsequent timesteps. 

Recently, RF-inversion~\cite{rout2025semantic} proposes an optimal-control based inversion for flow models, in which a guiding vector field steers the dynamics toward high-likelihood samples via velocity interpolation.
Unfortunately, it still involves computational overhead due to the ODE inversion. To address this, FlowEdit~\cite{kulikov2024flowedit}  proposed simulating an ODE between two image samples without inversion. However, 
empirical results showed that the method is quite sensitive to the hyperparameters and often fails to retain the source consistency.
Moreover,
the method heuristically applies classifier-free-guidance (CFG) to the both source and target velocity fields with different scaling factor and relies on skipping early timesteps of the ODE to enhance editing quality. These design choices introduce multiple hyperparameters to be searched and compromise the deterministic nature of ODE (see Section~\ref{subsec:recon}).

Motivated by the observation that these limitations of FlowEdit arise from nonsmooth and unstable editing trajectories—partly due to the lack of explicit latent inversion—we introduce {\em FlowAlign}, an optimal control-based  inversion-free approach for consistent and controllable text-driven image editing through trajectory regularization.
In contrast to  RF-inversion~\cite{rout2025semantic} that requires an ODE inversion,
we introduces a structural similarity at the terminal point as our trajectory regularization term to overcome the instability caused by the absence of the inverted latent.
%
Although the regularization is primarily enforced at the terminal point,
we further observe that it also naturally enforces source consistency along the trajectory by penalizing unnecessary deviations from the original image. 

Our method is also computationally efficient, requiring only one additional function evaluation (NFE) per ODE step for the regularization term. In contrast, methods like FlowEdit~\cite{kulikov2024flowedit} incur twice the NFE due to the reliance on classifier-free guidance. Despite this efficiency, FlowAlign achieves superior source preservation and competitive or improved editing quality compared to existing approaches
in image, video, and 3D editing. Finally, FlowAlign supports editing through backward ODE with the learned flow field, enabling accurate reconstruction of the original image from the edited output. This highlights the reversible and deterministic nature of the learned transformation, made possible by our explicit trajectory regularization.

\section{Backgrounds}
Suppose  we have access to samples from two distributions $X_1\sim p$ and $X_0\sim q$ that forms independent coupling, $\pi_{0,1}(X_0, X_1)=p(X_1)q(X_0)$. 
We can define a time-dependent function called flow $\psi_t(\vx) : [0,1]\times \mathbb{R}^d \rightarrow \mathbb{R}^d$, which is diffeomorphism and satisfies $\psi_t(X_1)=X_t$, to describe a continuous transform between $X_1$ and $X_0$.
Specifically, for $0\leq s < t \leq 1$, we can sequentially transfer the sample $X_t$ to $X_s = \psi_s(X_1)=\psi_s(\psi_t^{-1}(X_t))=\psi_{s|t}(X_t)$.
%
Here, the $\psi_t$ is uniquely characterized by a flow ODE
\begin{align}
    d\psi_t(\vx) = \vv_t(\psi_t(\vx))dt
    \label{eqn:flowode}
\end{align}
where $\vv_t$ represents velocity field.
Here, $\vv_t$ is approximated by a neural network $\vv_t^\theta$ with the parameter $\theta$ to construct a flow as generative model.
The training objective for the parameterized velocity field is called flow matching, which is expressed as
\begin{align}\label{eq:fm}
    \mathcal{L}_{FM} = \mathbb{E}_{t\in [0,1], \vx_t \sim p_t} \| \vv_t(\vx_t) - \vv_t^\theta(\vx_t) \|^2.
\end{align}
Unfortunately,  we cannot access $\vv_t(\vx_t)$ due to intractable integration over all $\vx_0$.
To address this, \cite{lipman2023flow} proposes conditional flow matching: %
%
%
\begin{align}\label{eq:cfm}
    \mathcal{L}_{CFM} = \mathbb{E}_{t\in [0,1], \vx_0 \sim q} \| \vv_t(\vx_t|\vx_0) - \vv_t^\theta(\vx_t) \|^2
\end{align}
where $\vv_t(\vx_t|\vx_0)$ is a conditional velocity field given by
\begin{align}\label{eq:condv}
    \vv_t(\vx_t|\vx_0) = 
    \dot\psi_t(\psi_t^{-1}(\vx_t|\vx_0)|\vx_0) =  \dot\psi_t(\vx_1 |\vx_0) ,
\end{align}
where the last equality is due to the definition of $\vx_t=\psi_t(\vx_1|\vx_0)$.
One of the most important contribution of \cite{lipman2023flow} is that the trained velocity field $\vv_t^\theta(\vx_t)$ from the conditional flow matching
\eqref{eq:cfm} can also approximate the {\em unconditional} velocity field $\vv_t(\vx_t)$ in \eqref{eq:fm}. Accordingly,
after training of the velocity field using \eqref{eq:cfm}, we can generate a sample by solving the following ODE:
\begin{align}\label{eq:flowode}
    d\vx = \vv_t^\theta(\vx_t) dt.
\end{align}
%
Among various flows, the affine conditional flow defined as $\psi_t(\vx_1|\vx_0) = a_t \vx_0 + b_t \vx_1$ is widely used, where $a_0=b_1=1$ and $a_1=b_0=0$. The resulting  conditional velocity field derived from the rectified flow for \eqref{eq:cfm} is then given by
\begin{align}
    \vv_t(\vx_t|\vx_0) = \dot \psi_t(\psi_t^{-1}(\vx_t|\vx_0)|\vx_0) = \dot a_t \vx_0 + \dot b_t \vx_1.
\end{align}
In case of linear conditional flow (or rectified flow) \cite{liu2023flow}, $a_t=t$ and $b_t=1-t$, and the conditional velocity field derived from the rectified flow for \eqref{eq:cfm} is $\vv_t(\vx_t|\vx_0)= \vx_1 - \vx_0$.

Without loss of generality, the problem formulation could be extended to a flow defined in latent space. Accordingly, we will use $\vx_t$ to denote both images and latent codes.

\begin{figure}[t]
    \centering
    \includegraphics[width=\linewidth]{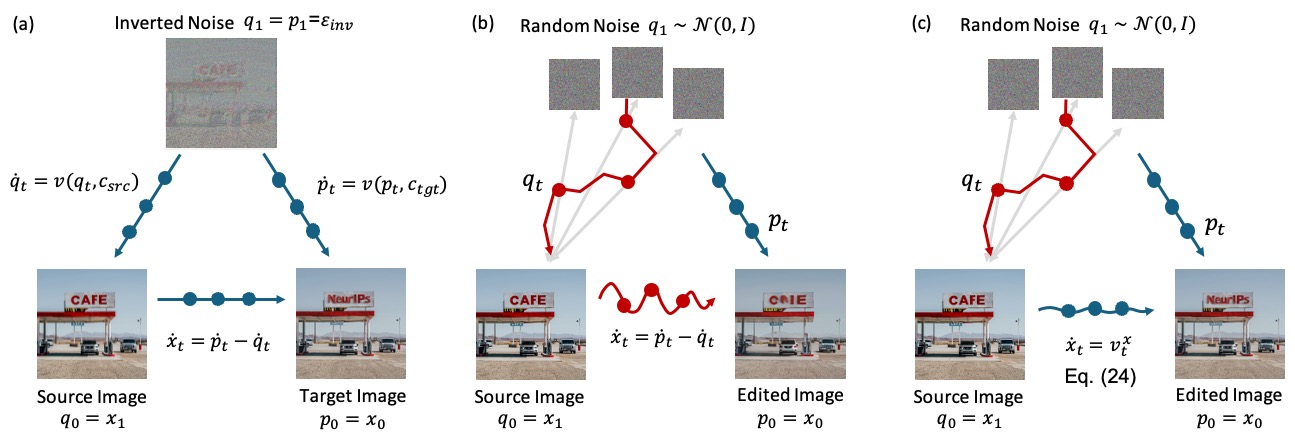}
    \vspace{-0.5cm}
    \caption{\textbf{Overview}. (a) Starting from the inverted latent,
     the ODE from source to target images can be obtained. 
     (b) In contrast, existing inversion-free approaches suffer from nonsmooth trajectories, as $\vq_t$ is sampled with random noise at each step, often resulting in editing artifacts.
(c) FlowAlign uses the regularized velocity $\vv_t^{\vx}$ from the similarity regularization at the terminal point, producing smoother and more consistent trajectories between the source and target images.
     Prompt:  {\em "…white and red sign that reads CAFE" $\rightarrow$ "…white and red sign that reads NeurIPS".}}
    \label{fig:overview}
\end{figure}

\section{FlowAlign}
\subsection{Text-based Image Editing using Pre-trained Flow Models}
In  text-based image editing using flow models, we aim to translate a source image $\vx_{src}$ at $t=1$ to a target image $\vx_{tgt}$ at $t=0$ based on text description of each image or editing instruction.
In particular, such translation can be represented through a linear conditional flow between two image distributions,
\begin{align}
    \psi_t(\vx_{src}|\vx_{tgt}) :=\vx_t = (1-t) \vx_{tgt} + t \vx_{src} \quad .
    \label{eqn:target_flow}
\end{align}
%
Note that the associated flow ODE described by \eqref{eq:flowode}
requires  training between two image distributions using the conditional flow matching
of \eqref{eq:cfm} with \eqref{eqn:target_flow}.
To bypass the additional training,
 consider two generative flows 
  that transfers the same noise $\epsilonb\sim\mathcal{N}(0, \rmI_d)$ to each image,
\begin{align}
    \label{eqn:forward_q}
     \psi_t^{src}(\epsilonb|\vx_{src}):=\vq_t &= (1-t) \vx_{src} + t \epsilonb, \\
     \label{eqn:forward_p}
     \psi_t^{tgt}(\epsilonb|\vx_{tgt}):=\vp_t &= (1-t) \vx_{tgt} + t \epsilonb,
\end{align}
and the associated ODEs:
\begin{align}\label{eq:aODE}
d\vq_t = \vv_t^{src}(\vq_t) dt ,\quad d\vq_t = \vv_t^{tgt}(\vp_t) dt 
\end{align}
with $\vq_1=\vp_1=\epsilonb$, $\vq_0 = \vx_{src}$, and $\vp_0=\vx_{tgt}$.
Since both flows are accessible without additional training by using a pre-trained noise-to-image  flow model\footnote{In this paper, we use Stable Diffusion 3.0 (medium), a foundational pre-trained flow model for the main experiments.},
it is beneficial if we can leverage \eqref{eqn:forward_q} and \eqref{eqn:forward_p} to simulate the  ODE for the flow in \eqref{eqn:target_flow}.

Importantly, the system of equations (\eqref{eqn:target_flow}, \eqref{eqn:forward_q} and \eqref{eqn:forward_p}) lead to the following key equality:
\begin{align}
    \vx_t = \vp_t -\vq_t + \vx_{src}, \quad  \mbox{where}\quad \vx_1=\vx_{src}, \vx_0=\vx_{tgt}
    \label{eqn:constraint}
\end{align}
Consequently, we can simulate the  ODE for image editing by
\begin{align}\label{eqn:target_ode}
    d\vx_t = d\vp_t - d\vq_t = [\vv_t^{tgt}(\vp_t) -\vv_t^{src}(\vq_t)]dt, \quad \mbox{where}\quad \vp_t:= \vq_t+\vx_t-\vx_{src}
\end{align}
without additional training of flow models (see Figure~\ref{fig:overview}(a)). %
In case of text-conditional pre-trained flow models, we can leverage the same neural network $\vv_t^\theta$ with text embeddings $c_{src}$ and $c_{tgt}$   to approximate the $\vv_t^{tgt}(\vp_t), \vv_t^{src}(\vq_t)$, respectively, leading to  
\begin{align}
    d\vx_t = [\vv_t^\theta(\vp_t, c_{tgt}) - \vv_t^\theta(\vq_t, c_{src})]dt.
    \label{eqn:simulation}
\end{align}

\subsection{Trajectory Errors from  Inversion-Free Approaches}

In \eqref{eqn:simulation}, $\vp_t$ and $\vq_t$ are first computed at each discrete time step $t$. If an inverted latent $\epsilonb_{\text{inv}}$ is obtained via ODE inversion from the source image $\vx_{\text{src}}$ along the trajectory of $\vq_t$ (as illustrated in Figure~\ref{fig:overview}(a)), then $\vp_t$ and $\vq_t$ can be further updated as:
%
\begin{align}\label{eq:pq}
 d\vq_t = \vv_t^\theta(\vq_t,c_{src}) dt,\quad d\vp_t = \vv_t^\theta(\vp_t,c_{tgt}) dt
\end{align}
with  the initialization condition $\vp_1=\vq_1=\epsilonb_{inv}$.
However, ODE inversion increase the overall computational cost.
Thus,  we are interested in 
avoiding the ODE inversion process. 
Toward this aim, FlowEdit~\cite{kulikov2024flowedit} proposed the following update for $\vp_t$ and $\vq_t$
\begin{align}\label{eq:flowedit}
 \vq_t = (1-t)\vx_{src} + t \epsilonb, \quad d\vp_t = \vv_t^\theta(\vq_t+\vx_t-\vx_{src},c_{tgt}) dt
\end{align}
with the initialization condition $\vp_1=\vq_1=\epsilonb$ for randomly sampled $\epsilonb$ at each $t$,
where the second equality stems from  \eqref{eqn:constraint}, i.e. $\vp_t=\vq_t+\vx_t-\vx_{src}$.

Unfortunately,  one of the critical limitations of FlowEdit using \eqref{eq:flowedit} is that
$\vq_t$ trajectory is not sufficiently smooth due to the random sampling of $\epsilonb$ (Figure~\ref{fig:overview}(b)).
 This leads to the inaccuracy of using simulated ODE in \eqref{eqn:simulation} for $\vx_t$.
To mitigate this, we propose an explicit trajectory regularization method using the structural similarity at the terminal point.

\subsection{Trajectory Regularization using Similarity at Terminal Point}

Given that the trajectory of $\x_t$ in \eqref{eqn:simulation} is not necessarily smooth due to the random sampling of $\epsilonb$, our objective is to simulate an ODE that explicitly penalizes deviations from a smooth trajectory as a form of regularization. Inspired by the recent advances of the optimal control approaches for flow models \cite{rout2025semantic, rout2025rbmodulation}, we consider the following time-reversal optimal control problem
\begin{align}
\dot\vx_t &= \vu(\vx_t),\quad \vx_1 =\vx_{src} \label{eq:ocx} \\
 V(\vu_t) &=  \int_{0}^1  \ell (\vx_t,\vu_t,t) dt + m (\vx_0) \label{eq:ocv}
\end{align}
In this paper, we use the following loss:
\begin{align}
\ell (\vx_t,\vu_t,t) & := \frac{1}{2}\|\vu_t-\left(v_t^\theta(\vp_t, c_{tgt}) - \vv_t^\theta(\vq_t, c_{src})\right)\|^2 , 
 \label{eq:ell}
\end{align}
thereby enforcing the similarity of the optimal control $\vu_t$ to the
 original velocity field in \eqref{eqn:simulation}.
 Unfortunately,  this is shown not sufficient  for regularizing the trajectory deviation owing to the lack of inverted latent.
 Therefore, our goal is to utilize the terminal loss $m(\vx_0)$ to enforce the trajectory smoothness.
 
One of the most important contributions of this work is showing that a similarity regularization at the terminal point serves the goal.
 Specifically, we introduce a
 $l_2$-based terminal point regularization:
 \begin{align}
 \m(\vx_0)=\frac{\eta}{2}\|\vx_0-\vx_{{src}}\|^2 \quad \label{eq:m}
 \end{align}
 This implies that the ODE evolution from 
from $t=1$ to $t=0$ ultimately converges to a terminal solution that closely resembles the starting point,
 i.e. $\vx_{tgt}\simeq \vx_{src}$.
 While this might raise concerns about convergence to a trivial solution where $\vx_t=\vx_{src}$ for all $t$,
the crucial distinction lies in our use of a finite 
$\eta$, as opposed to taking the limit $\eta\rightarrow \infty$ 
 as done in RF-inversion~\cite{rout2025semantic}.
%
 In what follows, we show that this terminal point regularization lead to the balance between the semantic guidance
 and structural preservation.
\begin{restatable}[]{prop}{rectflow}
\label{prop:rectflow}
    For the linear conditional flow with $a_t = 1-t$ and $b_t=t$ where $0\leq t \leq 1$, the ODE  that solves
    the optimal control problem with \eqref{eq:ocv}, \eqref{eq:ell} and \eqref{eq:m} is given by
\begin{align}
    d \vx_t &= \vv_t^\vx(\vp_t,\vq_t,\vp_0, \vq_0) dt
    \end{align}
    with the initial condition $\vx_1=\vx_{src}$, where
    \begin{align}
    \vv_t^\vx(\vp_t,\vq_t,\vp_0, \vq_0) &:\simeq \vv_t(\vp_t, c_{tgt}) - \vv_t(\vq_t, c_{src}) + \gamma
    \left(\mathbb{E}[\vp_0|\vp_t]-\mathbb{E}[\vq_0|\vq_t]\right) 
\end{align}
where $ \gamma = \frac{-\eta}{1-\eta t}$ is positive for sufficiently large $\eta$ and $\mathbb{E}[\vq_0|\vq_t] = \vq_t - t \vv_t(\vq_t, c_{src}),  \mathbb{E}[\vp_0|\vp_t] = \vp_t - t \vv_t(\vp_t, c_{tgt})$ are Tweedie's denoising estimates.
\end{restatable}
Note that the resulting velocity field can be decomposed as
\begin{align}
  \vv_t^\vx(\vp_t,\vq_t,\vp_0, \vq_0)  := \underbrace{[\vv_t(\vp_t, c_{tgt}) - \vv_t(\vq_t, c_{src})]}_{\text{Semantic Guidance}} + \gamma 
  \underbrace{(\mathbb{E}[\vp_0|\vp_t] - \mathbb{E}[\vq_0|\vq_t])}_{\text{Source Consistency}}.
   \label{eqn:edit_reg}
\end{align}

The final velocity in \eqref{eqn:edit_reg} needs more discussion. Similar to the drift term in \eqref{eqn:simulation},
the first term in \eqref{eqn:edit_reg}
serves as the directional signal from the source to the target semantics.
On the other hand, the second term
 becomes  a source consistent regularization gradient based on the distance between the clean estimates of $\vp_t$ and $\vq_t$, computed using Tweedie formula.
This regularization gradient keeps the trajectory $\vp_t$ close to the source-consistency direction represented by $\mathbb{E}[\vq_0|\vq_t]$, thereby implicitly regulating $\vx_t$.
%

Another byproduct of trajectory regularization is the robustness to the classifier-free guidance (CFG).
Unlike recent work~\cite{kulikov2024flowedit} that uses CFG to both estimated noises in with null-text embedding and different scales, 
we use the CFG as for the ODE trajectory $\vp_t$ only, 
\begin{align}
\vv^\theta(\vp_t,c_{src},c_{tgt}) = \vv^\theta(\vp_t, c_{src}) + \omega\left[\vv^\theta(\vp_t, c_{tgt}) - \vv^\theta(\vp_t, c_{src})\right]
\label{eqn:our_cfg}
\end{align}
where $\omega \geq 0$ is the CFG scale factor.
%
We find that commonly used CFG values (such as 7.5) are effective. For further analysis of the effect of $\omega$, see Section \ref{sec:abl}.
Accordingly, the proposed method is presented as time-reversal regularized trajectory:
\begin{align}
    d\vx_t = [\vv_t(\vp_t, c_{tgt}, c_{src}) - \vv_t(\vq_t, c_{src})]dt - \gamma dt (\mathbb{E}[\vq_0|\vq_t]-\mathbb{E}[\vp_0|\vp_t]),\quad \vx_1=\vx_{src}
    \label{eqn:final}
\end{align}
The complete algorithm is presented in Algorithm~\ref{alg:method}. Our approach introduces two hyperparameters $\omega$ and $\zeta=-\gamma dt >0$. We find that using constant values yields stable results and provide detailed analysis of their effects in Section \ref{sec:abl}.

\begin{algorithm}[!t]
\caption{Algorithm of FlowAlign on Latent Space}\label{alg:method}
\begin{algorithmic}[1]
\Require Source image $\vx_{src}$, Pre-trained flow model $\vv^\theta$, VAE encoder and Decoder $\mathcal{E}, \mathcal{D}$, 
Source/Target text embeddings $c_{src}, c_{tgt}$, CFG scale $\omega$, source consistency scale $\zeta$
\State $\vx_{src} \gets \mathcal{E}(\vz_{src})$
\State $\vx_t \gets \vx_{src}$
\For{$t: 1\rightarrow 0$}
    \State $\epsilon \sim \mathcal{N}(0, \rmI)$
    \State $\vq_t \gets (1-t) \vx_{src} + t \epsilonb$
    \State $\vp_t \gets \vx_t - \vx_{src} + \vq_t$
    \State $\vv^\theta(\vp_t) := \vv^\theta(\vp_t, c_{src}) + \omega\left[\vv^\theta(\vp_t, c_{tgt}) - \vv^\theta(\vp_t, c_{src})\right]$
    \State $\vv^\theta(\vq_t) := \vv^\theta(\vq_t, c_{src})$
    \State $\mathbb{E}[\vp_0|\vp_t] \gets \vp_t - t \vv^\theta(\vp_t), \quad \mathbb{E}[\vq_0|\vq_t] \gets \vq_t - t \vv^\theta(\vq_t)$
    \State $\vx_t \gets \vx_t + \left[\vv^\theta(\vp_t) - \vv^\theta(\vq_t)\right] dt + \zeta (\mathbb{E}[\vq_0|\vq_t]-\mathbb{E}[\vp_0|\vp_t])$
\EndFor
\State $\vz_{edit} \gets \mathcal{D}(\vx_t)$
\end{algorithmic}
\end{algorithm}


\begin{figure}[!t]
    \centering
    \includegraphics[width=0.9\linewidth]{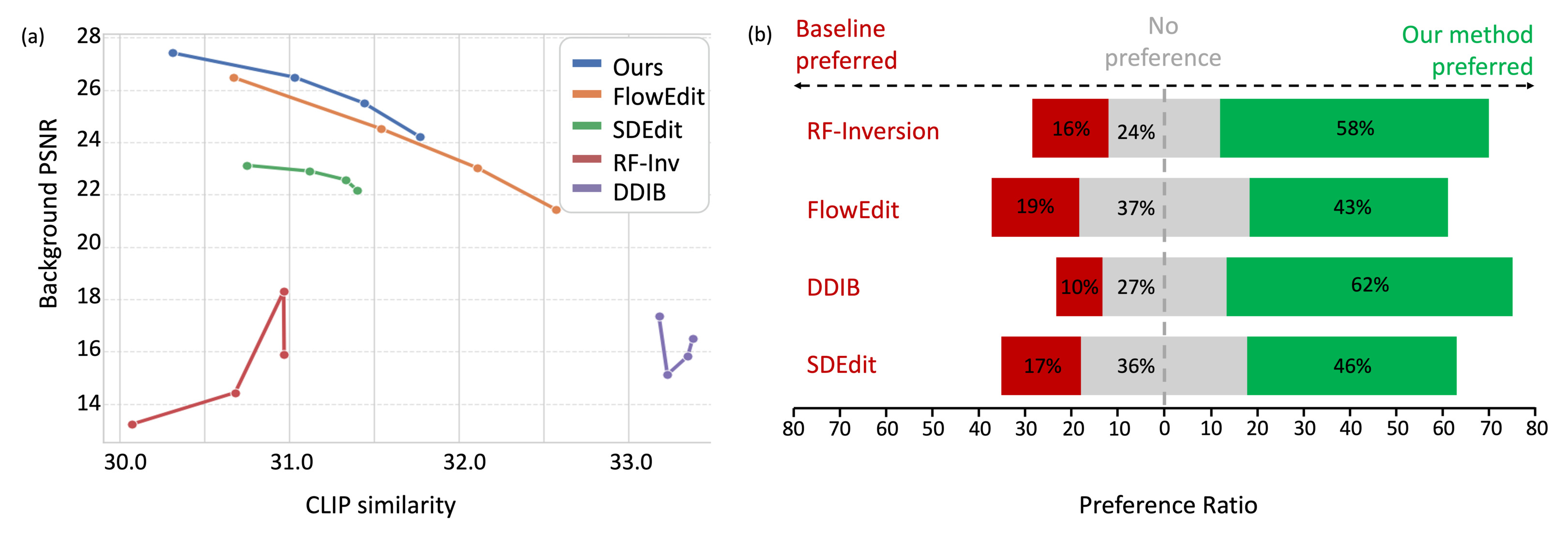}
    \vspace{-0.3cm}
    \caption{\textbf{Quantitative and human preference evaluation}. (a) Trade-off between CLIP similarity versus background PSNR. 
    (b) User preference results. Each bar shows the proportion of responses favoring the baseline (red), showing no preference (gray), or favoring our method (green). Bars are centered at 0 to emphasize directional preference. Across all comparisons, our method is preferred. 
    }
    \label{fig:metric}
\end{figure}

\section{Experimental Results}
\noindent\textbf{Dataset, Baseline and Compute Resource.}
To verify the editing performance, we perform multiple analysis using PIE-Bench~\cite{DBLP:journals/corr/abs-2310-01506} that contains 700 synthetic and natural images with paired original and editing prompts.
For baseline algorithms, we focus on comparing with methods that establish trajectory between two samples. For the methods that utilizes noisy sample distribution to connect two trajectories, we use SDEdit~\cite{meng2021sdedit} and DDIB~\cite{su2022dual}. For the method that improves the inversion process, we use RF-inversion~\cite{rout2025semantic}. For the inversion-free method, we select FlowEdit~\cite{kulikov2024flowedit}. Note that all these methods are training-free text-based image editing algorithms. For a fair comparison, we use the same flow model and 33 NFEs by following FlowEdit. In case of methods that involves inversion process, we use 17 NFEs for each inversion and sampling process. 
We conduct experiments using NVIDIA GeForce RTX 4090 (24GB VRAM). For runtime analysis, see Appendix~\ref{sec:runtime}. 

\begin{figure}[!t]
    \centering
    \includegraphics[width=\linewidth]{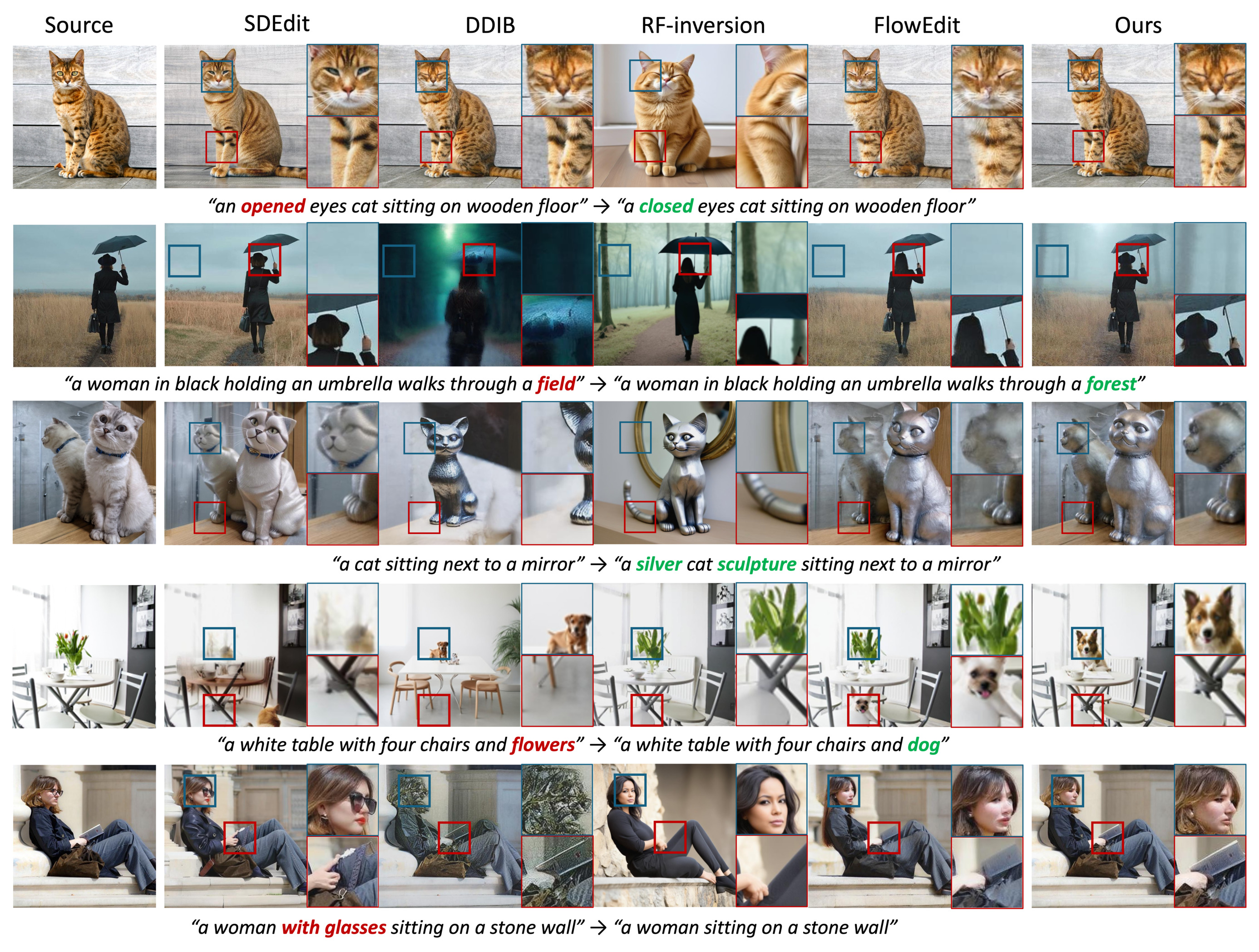}
    \vspace{-0.6cm}
    \caption{Qualitative comparison of text-based image editing methods. Insets provide zoomed-in views of regions highlighted by red and blue rectangles. Our method achieves better semantic alignment and structure consistency across a diverse set of prompts.}
      \vspace{-0.3cm}
    \label{fig:qual}
\end{figure}

\noindent\textbf{Semantic alignment and Structure Consistency.}
Image editing quality should be evaluated from two complementary perspectives: semantic alignment and source structure consistency.
An effective editing method should achieve a balance between these two objectives, improving both simultaneously.
To assess our method, we follow prior works and report CLIP similarity as a measure of semantic alignment and background PSNR as a proxy for structural consistency. We evaluate performance across various classifier-free-guidance (CFG) scales \{5.0, 7.5, 10.0, 13.5\}, as CFG is a key parameter that modulates the trade-off between semantic fidelity and structure consistency in flow-based, text-guided image editing.
Figure~\ref{fig:metric}a presents the quantitative evaluation results. For the complete results, please refer to Appendix~\ref{sec:complete}.
While the trade-off between semantic alignment and structure consistency remains, the proposed method consistently achieves higher structural preservation compared to all other methods.
In terms of semantic alignment, measured by CLIP similarity, the proposed method outperforms SDEdit and RF-Inversion, while FlowEdit (at certain CFG scales) and DDIB achieves higher scores.
However, these baselines tend to increase CLIP similarity at the cost of source structure consistency. In many cases, their high CLIP scores result from over-expression of target prompt objects, often distorting the original image as shown in Figure~\ref{fig:qual}.

%
    \begin{table}[!hbt]
    \centering
    \resizebox{0.7\linewidth}{!}{
    \begin{tabular}{lccccc}
    \toprule
         Metric & SDEdit~\cite{meng2021sdedit} & DDIB~\cite{su2022dual} & RF-Inv~\cite{rout2025semantic} & FlowEdit~\cite{kulikov2024flowedit} & Ours \\
         \midrule
         PSNR $\uparrow$ & 13.83 & 18.18 & 12.14 & 19.88 & \textbf{27.42} \\
         DINO Dist $\downarrow$ & 0.078 & 0.041 & 0.113 & 0.037 & \textbf{0.025} \\
         LPIPS $\downarrow$ & 0.419 & 0.190 & 0.502 & 0.147 & \textbf{0.085} \\
         MSE $\downarrow$ & 0.043 & 0.019 & 0.065 & 0.012 & \textbf{0.006} \\
         \bottomrule
    \end{tabular}
    }
    \caption{Quantitative comparison for backward editing results. \textbf{Bold} represents the best.}
    \label{tab:recon_metric}
     \vspace{-0.3cm}
\end{table}

\noindent\textbf{Human evaluation.}
Due to limitations of current metrics, such as CLIP similarity (as discussed in the previous section), we additionally conduct a human preference study.

From a pool of 700 validation samples, we randomly select 100 images along with their corresponding original and editing text prompts. For each participant, we present a pairwise comparison between the edited image produced by the proposed method and that of a randomly selected baseline. 
Participants are asked to indicate which result they prefer, that is defined as one with accurate reflectance of editing prompt with source structure preservation. For a more detailed description of the human evaluation protocol, refer to Appendix~\ref{sec:human}.
Figure~\ref{fig:metric}b summarizes the preference ratio of out method against each baseline. Across all comparisons, the proposed method is more preferred. These results demonstrates that our method achieves improved editing fidelity and superior source structure consistency, as intended.

\noindent\textbf{{Backward Editing.}}
\label{subsec:recon}
In this work, we proposed a flow matching regularization to make smoother trajectory between two image samples. Thus, we investigate whether the proposed trajectory \eqref{eqn:final} behaves deterministically between two samples. To test this, we solve \eqref{eqn:final} in backward direction starting from the edited image.
By evaluating the similarity between the source image and reconstructed image, we can evaluate whether the editing method  indeed shows behavior like an ODE.
We compute both pixel-wise metrics (PSNR, MSE) and perceptual metrics (LPIPS, DINO structural distance) and report the results in Table~\ref{tab:recon_metric}. Across all metrics, the proposed method outperforms the baselines. These results suggest that reverse editing using the proposed method nearly reconstructs the source image, supporting the effectiveness of the proposed trajectory regularization.
Figure~\ref{fig:recon} shows a qualitative comparison for the edited image and reconstructed image. For example, in the last column, the proposed method uniquely reconstruct the torch of the statue to match the source image, whereas the baselines fail to reconstruct it accurately and instead convert it into a real torch. Importantly, the edited image obtained by the proposed method faithfully reflects the intended editing direction. This implies that the observed reconstruction ability stems from the smooth trajectory, rather than simply reducing changes during the editing process.

\begin{figure}[!t]
    \centering
    \includegraphics[width=\linewidth]{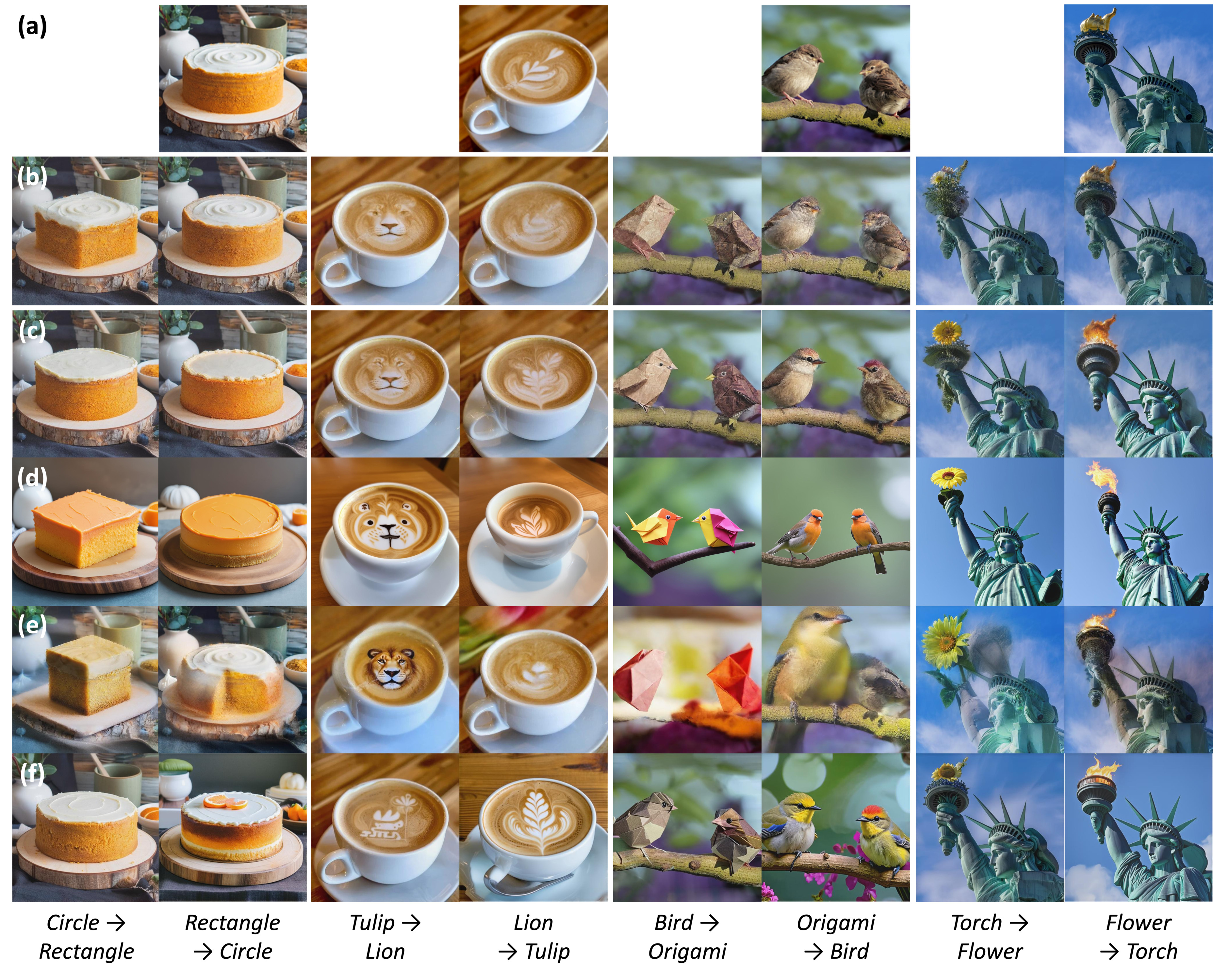}
    \vspace{-0.5cm}
    \caption{Qualitative comparison of editing (odd columns) and backward editing (even columns) results. FlowAlign (the 2nd row) achieves better reconstruction quality compared to the baselines. (a) Source image, (b) FlowAlign, (c) FlowEdit, (d) RF-inversion, (e) DDIB, and (f) SDEdit.}
    \label{fig:recon}
\end{figure}

\begin{wrapfigure}{r}{0.48\textwidth}
  \begin{center}
\vspace{-0.5cm}
    \includegraphics[width=0.4\textwidth]{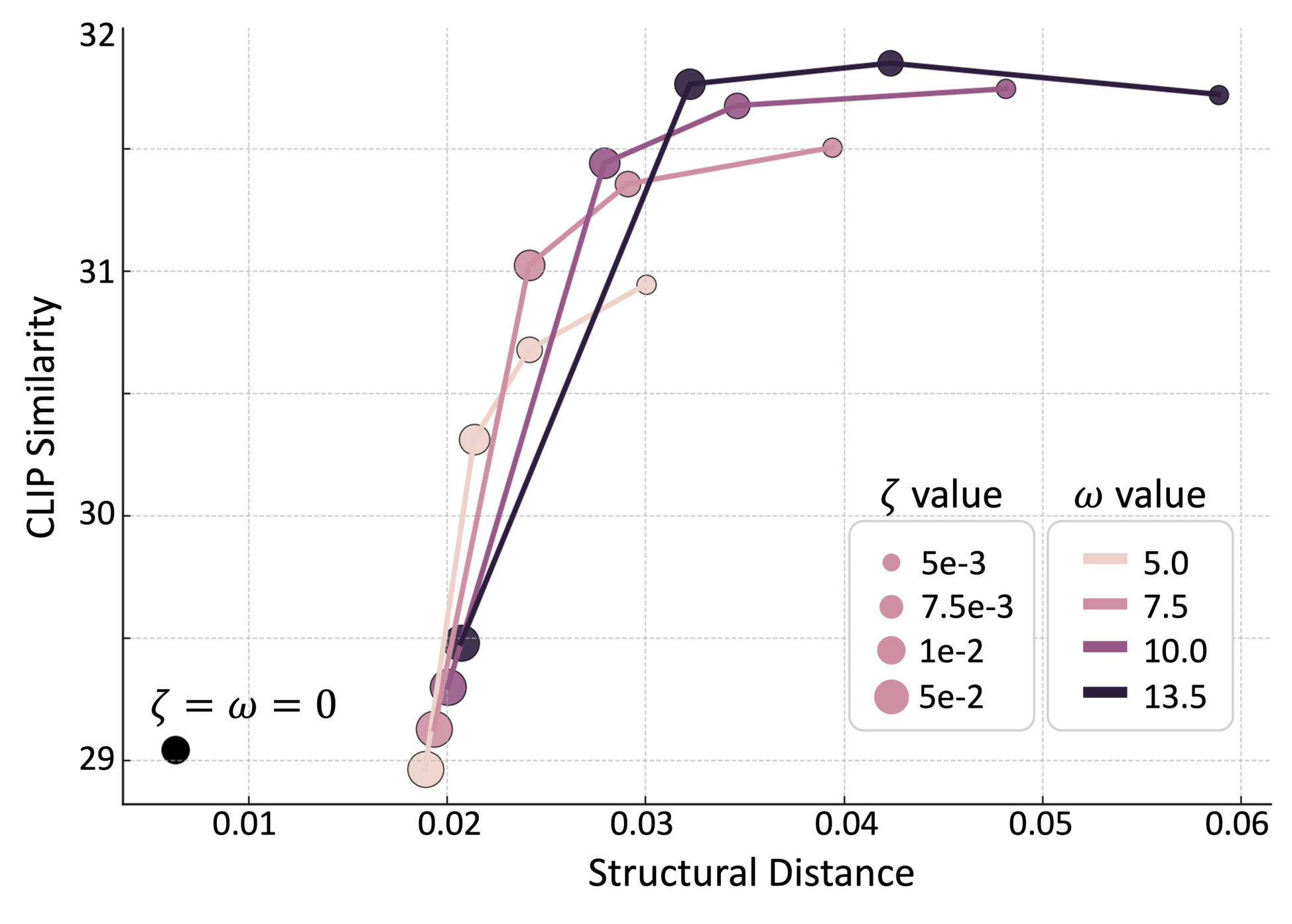}
  \end{center}
 \vspace{-0.6cm}
  \caption{Ablation study for $\omega$ and $\zeta$. Top-left points represent to balanced performance.}
  \vspace{-0.2cm}
  \label{fig:abl}
\end{wrapfigure}

\noindent\textbf{Ablation Study.}
\label{sec:abl}
The proposed method introduces two hyperparameters, $\omega$ and $\zeta$, which control the relative strength of each regularization term in \eqref{eqn:edit_reg}. These parameters govern the trade-off between semantic alignment with the target prompt and structural consistency with the source image.
To evaluate their impact, we conduct an ablation study under various settings. As shown in Figure~\ref{fig:abl}, we observe that setting $\zeta=0.01$ consistently achieves a favorable balance between semantic alignment and structure preservation.
Notably, the framework without any additional gradients (i.e. $\omega=\zeta=0$, block dot) results in insufficient editing, which appears at the left-bottom position. This result emphasizes the effectiveness of the proposed CFG.

\begin{figure}[!t]
    \centering
    \includegraphics[width=\linewidth]{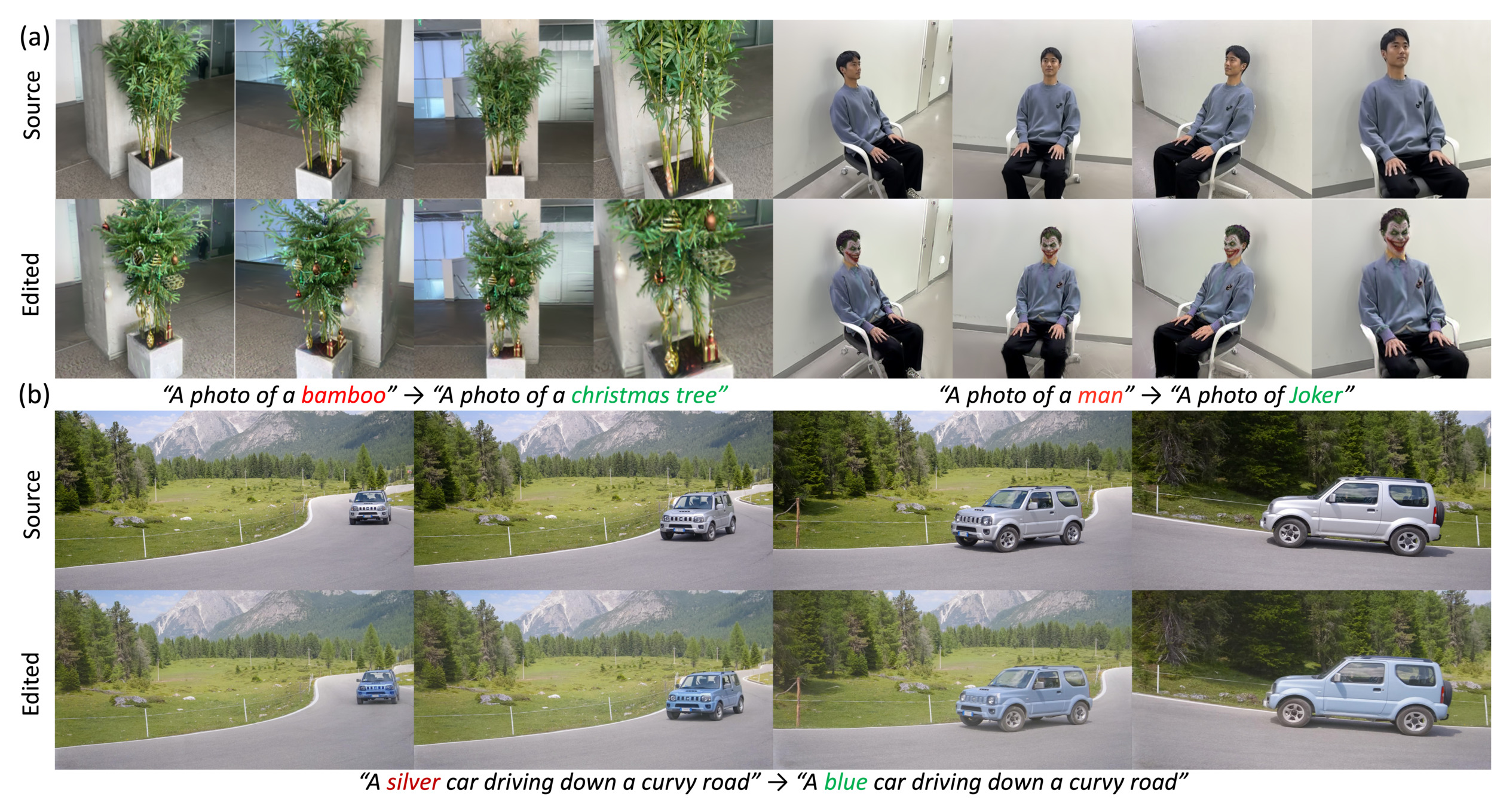}
    \vspace{-0.4cm}
    \caption{Editing results for (a) 3D Gaussian splatting rendered from four different viewpoints, and (b) a video sequence edited frame-by-frame.}
    \label{fig:further}
\end{figure}

\noindent\textbf{Further Applications.}
Figure~\ref{fig:further}a demonstrates the  application of FlowAlign to 3D editing via Gaussian splatting, highlighting its effectiveness in editing Gaussian parameters using FlowAlign in place of standard score distillation—thus extending its utility beyond 2D image editing.
Figure~\ref{fig:further}b illustrates the application of FlowAlign to video editing, where it is applied independently to each frame. Although temporal consistency is not explicitly enforced, the strong source consistency of FlowAlign results in visually coherent backgrounds across frames.
Additional details and results for these applications are provided in Appendix~\ref{sec:additional}.

\section{Conclusion}
In this paper, we propose a flow matching regularization, that leads to smooth and stable editing trajectory, for inversion-free flow-based image editing algorithm. By defining differentiable optimal control problem with similarity regularization at the terminal point, we explicitly balances semantic alignment with editing text and structural consistency with the source image. 
As a result, the proposed method achieves comparable or better editing performance while showing superior source consistency and better computational efficiency.
Notably, we demonstrate that the simulated ODE exhibits a deterministic property by performing the reverse editing experiments.
In summary, this work provides a novel design space for approximating ODEs between two samples without requiring additional training.

\noindent\textbf{Limitation and Potential Negative Impacts.}
While the proposed method offers efficient training-free image editing algorithm, it requires multiple diffusion timesteps for inference. 
Because the editing direction relies on the prior knowledge of pre-trained diffusion model, our method inherits potential negative impacts of generative models such as biased generation.


\bibliographystyle{plain}
\bibliography{reference}




\clearpage

\clearpage
\appendix

\begin{center}
  {\LARGE \bfseries Supplementary Material}\\[0.5em]
  {\large FlowAlign: Trajectory-Regularized,  Inversion-Free Flow-based Image Editing}
\end{center}
\vspace{1em}

\section{Optimal Control Formulation to Prove Proposition~\ref{prop:rectflow}}

The following lemma is required for the proof.

\begin{lemma}\label{lem:oc}
Consider the following time-reversal optimal control problem:
\begin{align}
 V(\vu_t) =  \int_{0}^1  \ell (\vx_t,\vu_t,t) dt + m (\vx_1),  & \quad \dot\vx_t = \vu(\vx_t) ,\quad \vx_{t_0} = \vx_{start},
\end{align}
where
\begin{align}
\ell (\vx_t,\vu_t,t) & := \frac{1}{2}\|\vu_t-\va_t \|^2  \\
m(\vx_1) &:= \frac{\eta}{2}\|\vx_0-\vx_{src}\|^2
\end{align}
Then, the optimal solution trajectory is given by
\begin{align}
\dot \vx_t = -\pb_t +\va_t 
\end{align}
where $\pb_t$ is given by
\begin{align}
\pb_t =   \frac{-\eta}{1-\eta t} \left(\vx_{t}+\vv_{t} -\vx_{src} \right) 
&\quad \mbox{where} \quad
 \vv_t:= \int_t^0 \va_t dt 
\end{align}
\end{lemma}
\begin{proof}
The Hamiltonian for the given optimal control problem can be represented by
\begin{align}
H(\vx_t,\vu_t,\pb_t,t):=  \frac{1}{2}\|\vu_t-\va_t \|^2  + \pb_t^T\ub_t
\end{align}
The optimal control $\vu_t^*$ that minimizes the Hamitonian is then given by
\begin{align}
\vu_t^*=\va_t-\pb_t
\end{align}
This leads to the following
\begin{align}
H(\vx_t,\vu^*_t,\pb_t,t) := & -\frac{1}{2}\|\vp_t\|^2+ \pb_t^T\va_t 
\end{align}
Then, the minimum principle \cite{basar2020lecture, fleming2012deterministic} informs that the optimal pair $(\vx_t,\vp_t)$ should satisfy the following:
\begin{align}
\dot \vx_t  &= \vu_t^*=\va_t-\pb_t
 \\
\dot \pb_t &= - \frac{\partial H(\vx_t,\vu^*_t,\pb_t,t) }{\partial \vx} = \mathbf{0} \label{eq:pb}
\end{align}
with the additional boundary conditions
\begin{align}\label{eq:pb1}
\vx_{t_0} = \vx_{start}, &\quad  \pb_0 =  \left. \frac{\partial m(\vx_t,t)}{\partial \vx_t}\right|_{t=0} = \eta (\vx_0 -\vx_{src}) 
\end{align}
From \eqref{eq:pb}, we can see that $\pb_t$ is time-invariant constant, i.e. $\pb_t=\pb$ for all $t\in [0,1]$.
Accordingly, we have
\begin{align}
\vx_0 &= \vx_{t_0} + \int_{t_0}^0 -\pb +\va_tdt \\
&= t_0  \pb +  \vx_{t_0} +\vv_{t_0}
\end{align}
where 
\begin{align}
\vv_{t_0}:= \int_{t_0}^0 \va_t  dt 
\end{align}
By plugging this in \eqref{eq:pb1} with $\pb_0=\pb$, we have
\begin{align*}
\pb = \eta\left(  t_0  \pb +  \vx_{t_0} +\vv_{t_0}-\vx_{src} \right) &\quad \Rightarrow \pb = \frac{\eta(\vx_{t_0} +\vv_{t_0}-\vx_{src})}{1-t_0\eta}
\end{align*}
Therefore, the optimal control  $\vu_t^*$ is given by
\begin{align}
\vu_t^* 
&=  \va_t + \frac{-\eta}{1-\eta t} \left(\vx_{t}+\vv_{t} -\vx_{src} \right) 
\end{align}
\end{proof}

Now we are ready to prove our main results.
\rectflow*
\begin{proof}
We can now use  Lemma~\ref{lem:oc} with the following modification
\begin{align}
\va_t := \vv_t^\theta(\vp_t, \vc_{tgt}) - \vv_t^\theta(\vq_t, \vc_{src}) 
\end{align}
This leads to  the following first order approximation:
\begin{align}
\vv_{t}  &= \int_t^0 \va_t dt =  \int_t^0 \vv_t(\vp_t) -\vv_t(\vq_t) dt \\
&\simeq -t \vv_t(\vp_t) + t \vv_t (\vq_t)
\end{align}
Accordingly, we have
\begin{align*}
\vx_{t}+\vv_{t} -\vx_{src}   &= \vp_t-\vq_t  +\vv_t \notag\\
&\simeq  \vp_t-\vq_t + (-t \vv_t(\vp_t) + t \vv_t (\vq_t))  \\
&=\mathbb{E}[\vp_0|\vp_t]-\mathbb{E}[\vq_0|\vq_t]
\end{align*}
where we use  $\vx_{t} -\vx_{src} = \vp_t-\vq_t $ from \eqref{eqn:constraint} for the first
equality. 
Thus, the optimal control  $\vu_t^*$ is given by
\begin{align}
\vu_t^* 
&=  \va_t + \frac{-\eta}{1-\eta t} \left(\vx_{t}+\vv_{t} -\vx_{src} \right) \notag\\
&\simeq  \va_t +\frac{-\eta}{1-\eta t}\left(\mathbb{E}[\vp_0|\vp_t]-\mathbb{E}[\vq_0|\vq_t]\right) 
\end{align}
Therefore, we have
\begin{align}
\dot \vx_t &= \vv_t(\vp_t, \vc_{tgt}) - \vv_t(\vq_t, \vc_{src}) +\frac{-\eta}{1-\eta t} \left(\mathbb{E}[\vp_0|\vp_t]-\mathbb{E}[\vq_0|\vq_t]\right)\\
&= \vv_t(\vp_t, \vc_{tgt}) - \vv_t(\vq_t, \vc_{src}) +\gamma \left(\mathbb{E}[\vp_0|\vp_t]-\mathbb{E}[\vq_0|\vq_t]\right)
\label{eqn:final_drift}
\end{align}

\end{proof}

\section{Implementation details}
\label{sec:details}
In this section, we provide implementation details of baseline methods and the proposed method. For the proposed method, we will release the code to \url{https://github.com/FlowAlign/FlowAlign}.

\paragraph{Backbone model}
While various text-to-image generative models exist, we leverage a flow-based model defined in latent space, specifically using Stable Diffusion 3.0 (medium)~\cite{esser2024scaling} provided by the diffusers package.
Given the significantly improved generative and text alignment performance of this backbone compared to earlier versions, we use the same backbone model for all baselines to establish a fair comparison.
For the time discretizations, we set the shift coefficient to 3.0, which is a default option of the Stable Diffusion 3.0


\paragraph{Baseline methods}
\label{sec:baseline}
\begin{enumerate}
    \item DDIB \cite{su2022dual} : 
    DDIB involves an inversion process followed by a sampling process. For inversion, we adopt the backward flow ODE. Regarding classifier-free guidance (CFG), we use only the null-text embedding during inversion and both target text and null-text embeddings during the sampling. This choice is motivated by the instability observed when applying standard CFG (i.e., using the source text and null-text embeddings) during inversion. 
    To ensure a fair comparison under similar computational cost, we set the number of ODE timesteps to 17 for both the inversion and sampling processes.
    
    \item SDEdit \cite{meng2021sdedit} : 
    SDEdit requires specifying the initial SNR (i.e. timestep), and its performance can vary significantly depending on this choice. In this work, our main focus is to address the limitations of inversion-free editing methods. To fairly demonstrate the effectiveness of our approach in comparison to alternative methods that also can mitigate this issue, we adopt the same initial SNR setting as FlowEdit~\cite{kulikov2024flowedit}, determined by the starting timestep of the ODE. Specifically, we use the 18th timestep as the starting point in our experiments.

    \item RF-inversion \cite{rout2025semantic} :
    RF-Inversion introduces an optimal-control-based guidance mechanism that ensures the inverted representation aligns with a target terminal state, resuling in a sampling process that is more likely under a predefined terminal distribution. We follow the official implementation, setting $\gamma=0.5$, $\eta=0.9$, the starting time $s=0$, and the stopping time $\tau = 0.25$. RF-Inversion uses only the null-text embedding, while both the target text and null-text embeddings are used during the sampling phase.
    
    \item FlowEdit \cite{kulikov2024flowedit} :
    We follow the official implementation of FlowEdit, setting the CFG scale to 3.0 for the source direction and 13.5 for the target direction. Additionally, we solve the flow ODE starting from the 18th timestep out of 50, resulting in 33 ODE timesteps.
    
\end{enumerate}

\paragraph{Evaluation Metrics}
For the quantitative comparison, we evaluate following metrics using the official evaluation code \footnote{\url{https://github.com/cure-lab/PnPInversion/tree/main/evaluation}} from PIEBench~\cite{ju2024pnp}:
\begin{enumerate}
    \item Background PSNR : PIEbench~\cite{ju2024pnp} provides masks that cover the object to be edited. Accordingly, we compute the PSNR by excluding the masked region, resulting in the background PSNR.
    \item Background LPIPS : We measure the LPIPS~\cite{blau2018perception}, which is defined as distance between feature maps of pre-trained VGG network, by excluding the masked region.
    \item Background SSIM : We compute the structural similarity~\cite{wang2004image} by excluding the masked region.
    \item Background MSE : we compute pixel-wise mean-squared-error by excluding the masked region.
    \item CLIP-score : We report the similarity between features embedded by pre-trained CLIP~\cite{radford2021learning}\footnote{We use CLIP ViT-base-patch16.} image encoder and text encoder. For the CLIP score within the edited region, we apply it only to the masked area.
\end{enumerate}




\section{Human preference test protocol}
\label{sec:human}
To evaluate the quality of image editing, we conduct a human preference study in the form of an AB-test. Specifically, we randomly sample 100 images from the 700 validation samples of PIEBench~\cite{ju2024pnp}. The study follows the protocol below for each participant:

\begin{enumerate}
    \item Randomly select one sample from the 100-image pool.
    \item Randomly choose one baseline method from the four.
    \item Randomly assign the baseline and the proposed method to gruops A and B.
    \item Display AB-test user instruction with editing instruction, source image, and edited results from both methods.
    \item The participant selects one of the following options: "A is better", "B is better", or "Not sure".
    \item The participant clicks "Submit", and the response is recorded.
    \item Step 1-6 are repeated until 20 cases are completed.
    \item If more than half of the responses are "Not sure", an additional 5 comparisons is presented following the same protocol.
\end{enumerate}

Although the editing methods utilize source–target text pairs, we present the editing instructions from PIEBench to participants instead, aiming to improve readability and reduce cognitive load.
%
The AB-test interface is shown Fig.~\ref{fig:ab_ui}, and was implement using Google Script. We recruited 25 participants and collected a total of 505 responses (with one participant receiving additional cases due to frequent "Not sure" selections). These responses were used to compute the final human preference results.

\begin{figure}[t]
    \centering
    \includegraphics[width=0.8\linewidth]{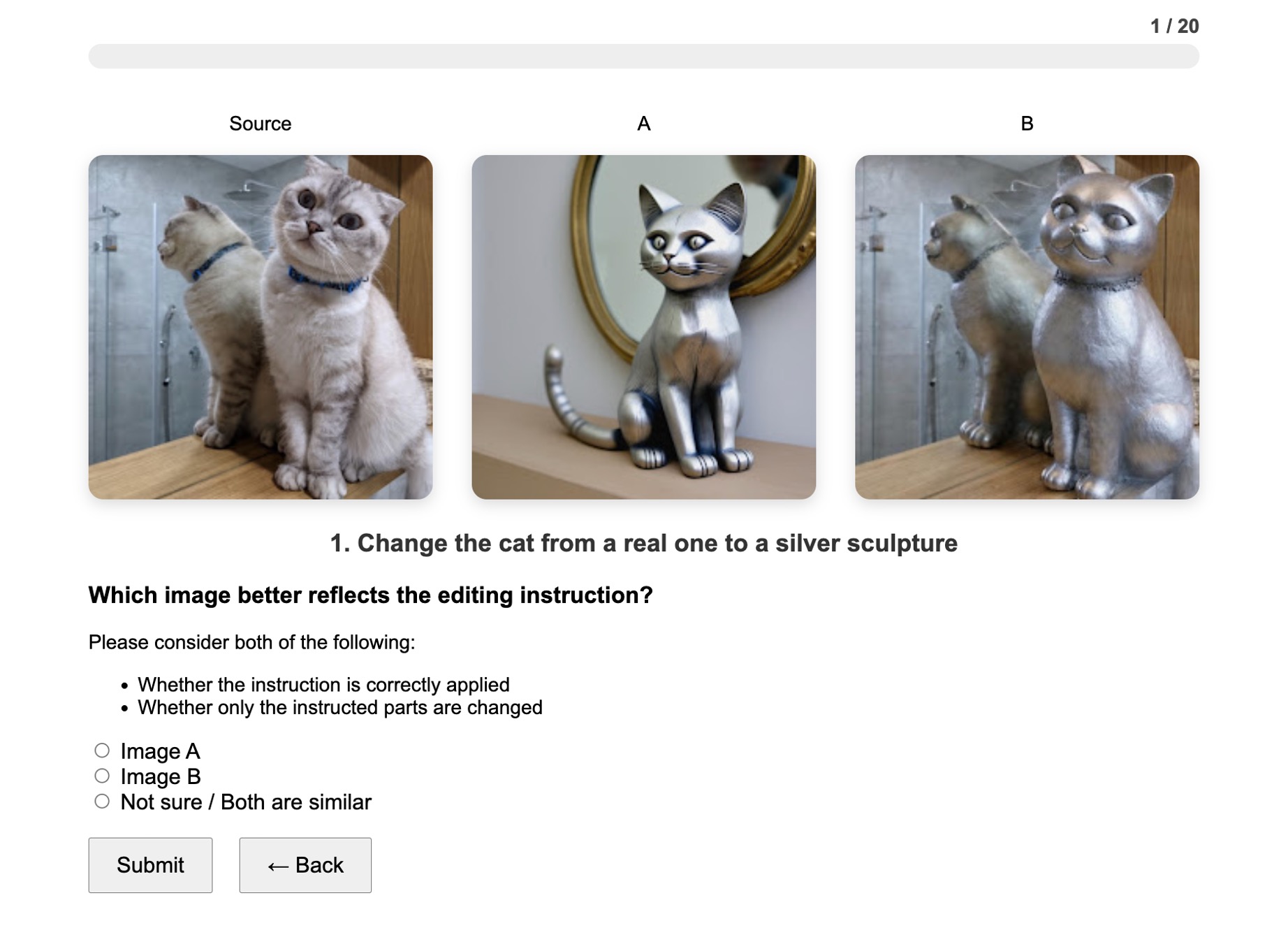}
    \caption{User interface for AB-test.}
    \label{fig:ab_ui}
\end{figure}

\section{Run-time comparison}
\label{sec:runtime}
In this section, we report the wall-clock time for each editing method evaluated in the main paper, specifically focusing on the runtime for solving ODE in latent space. We use a single RTX 4090 to measure the runtime.
For methods involving an inversion process, the number of function evaluations (NFEs) is effectively doubled compared to reverse sampling.
However, to ensure fair comparison, all experiments are conducted with the similar number of ODE timesteps where the sample $\vx_t$ is updated. Thus, for the inversion-based method, we update sample with 17 ODE timesteps for inversion and 17 ODE timesteps for the sampling.
Table~\ref{tab:runtime} shows the averaged runtime (in seconds) measured across 50 editing cases. 

For SDEdit, the runtime corresponds to standard reverse sampling with CFG. In comparison, DDIB and RF inversion achieve slightly shorter runtime, as we set the CFG scale to 0 during inversion. This is because applying CFG in the inversion stage often amplifies errors and disrupts the sampling process.
FlowEdit requires computing CFG twice per ODE timestep - once for the source direction and once for the target - which results in approximately double the runtime of SDEdit. In other words, when using the same discretized ODE schedule, FlowEdit incurs a similar computational cost to inversion-based methods.
In contrast, the proposed method achieves faster runtime due to its efficient CFG, which is computed only for $\vv(\vp_t)$. Importantly, this computational efficiency is achieved without sacrificing performance.


\begin{table}[H]
    \centering
    \resizebox{0.7\linewidth}{!}{
    \begin{tabular}{c|ccccc}
    \toprule
         & DDIB & SDEdit & RF-inv & FlowEdit & FlowAlign\\
    \midrule
         Stable Diffusion 3.0 & 4.06 & 5.26 & 4.06 & 10.64 & 7.14 \\
    \bottomrule
    \end{tabular}
    }
    \vspace{0.1cm}
    \caption{Runtime comparison (unit: seconds). Averaged time for 50 samples is reported.}
    \label{tab:runtime}
\end{table}


\section{Additional evaluation results on PIEBench}
\label{sec:complete}
\subsection{Quantitative and qualitative results}
We provide the complete evaluation result on PIEBench, which includes metrics described in Section~\ref{sec:details}.
As discussed in the main paper, the proposed method outperforms all baselines in source consistency metrics while achieving a CLIP score comparable to FlowEdit. Although inversion-based editing algorithms such as DDIB and RF-Inversion exhibit higher CLIP scores, they tends to overemphasize the target concept at the expense of preserving the source structure (see more examples in Fig.~\ref{fig:supple_qualitative0} and \ref{fig:supple_qualitative1}).
We also illustrate more qualitative results of the proposed method on text-based image editing task.
Specifically, Fig.~\ref{fig:supple_qualitative0} and \ref{fig:supple_qualitative1} present additional qualitative comparison between baselines and the FlowAlign. For diverse editing categories and objects, the proposed method shows better editing capability with better source structure preservation.

\begin{table}[t]
    \centering
    \resizebox{\linewidth}{!}{
    \begin{tabular}{lccccccc}
    \toprule
         & \multicolumn{5}{c}{Source Consistency} & \multicolumn{2}{c}{Semantic Alignment}\\
    \cmidrule(lr){1-1} \cmidrule(lr){2-6} \cmidrule(lr){7-8}
         method & Structural Distance $\downarrow$ & Background PSNR $\uparrow$ & Background LPIPS $\downarrow$& Background MSE $\downarrow$ & Background SSIM $\uparrow$ & CLIP Entire $\uparrow$& CLIP Edited $\uparrow$\\
         \midrule
         DDIB~\cite{su2022dual} & 0.106 & 14.43 & 0.262 & 0.037 & 0.746 & 30.68 & 26.88 \\
         SDEdit~\cite{meng2021sdedit} & 0.036 & 22.57 & 0.119 & 0.008 & 0.747 & 24.56 & 21.95 \\
         RF-Inv~\cite{rout2025semantic} & 0.103 & 15.82 & 0.291 & 0.028 & 0.663 & 33.35 & 29.60 \\
         FlowEdit~\cite{kulikov2024flowedit} & 0.036 & 23.02 & 0.082 & 0.007 & 0.842 & 25.98 & 22.81 \\
         FlowAlign &  0.028 & 25.50 & 0.053 & 0.004 & 0.879 & 25.28 & 22.00 \\
    \bottomrule
    \end{tabular}
    }
    \caption{Quantitative results for image editing on PIEBench (CFG scale 10.0).}
    \label{tab:my_label}
\end{table}

\begin{figure}[t]
    \centering
    \includegraphics[width=\linewidth]{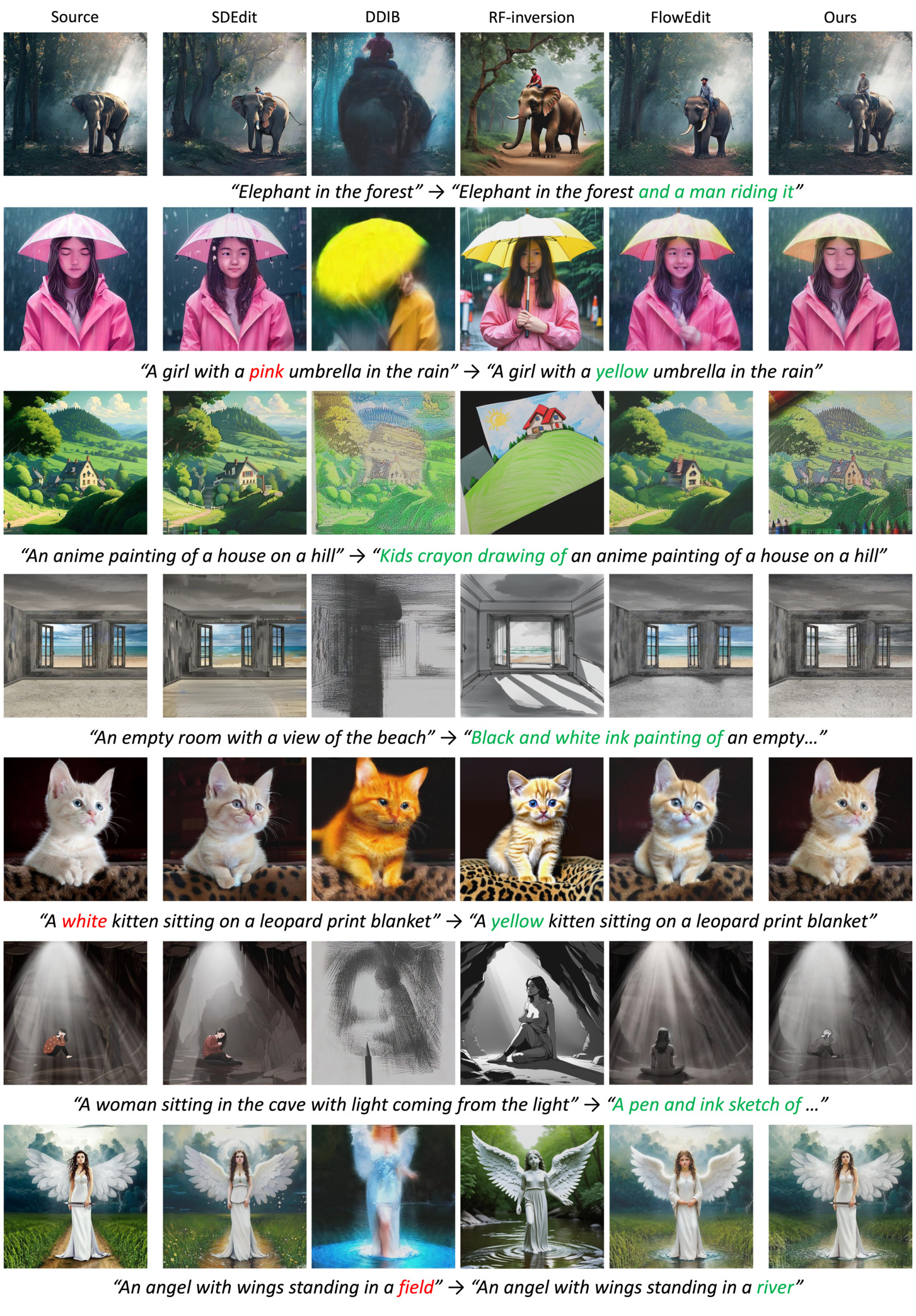}
    \caption{Additional qualitative comparison results.}
    \label{fig:supple_qualitative0}
\end{figure}

\begin{figure}[t]
    \centering
    \includegraphics[width=\linewidth]{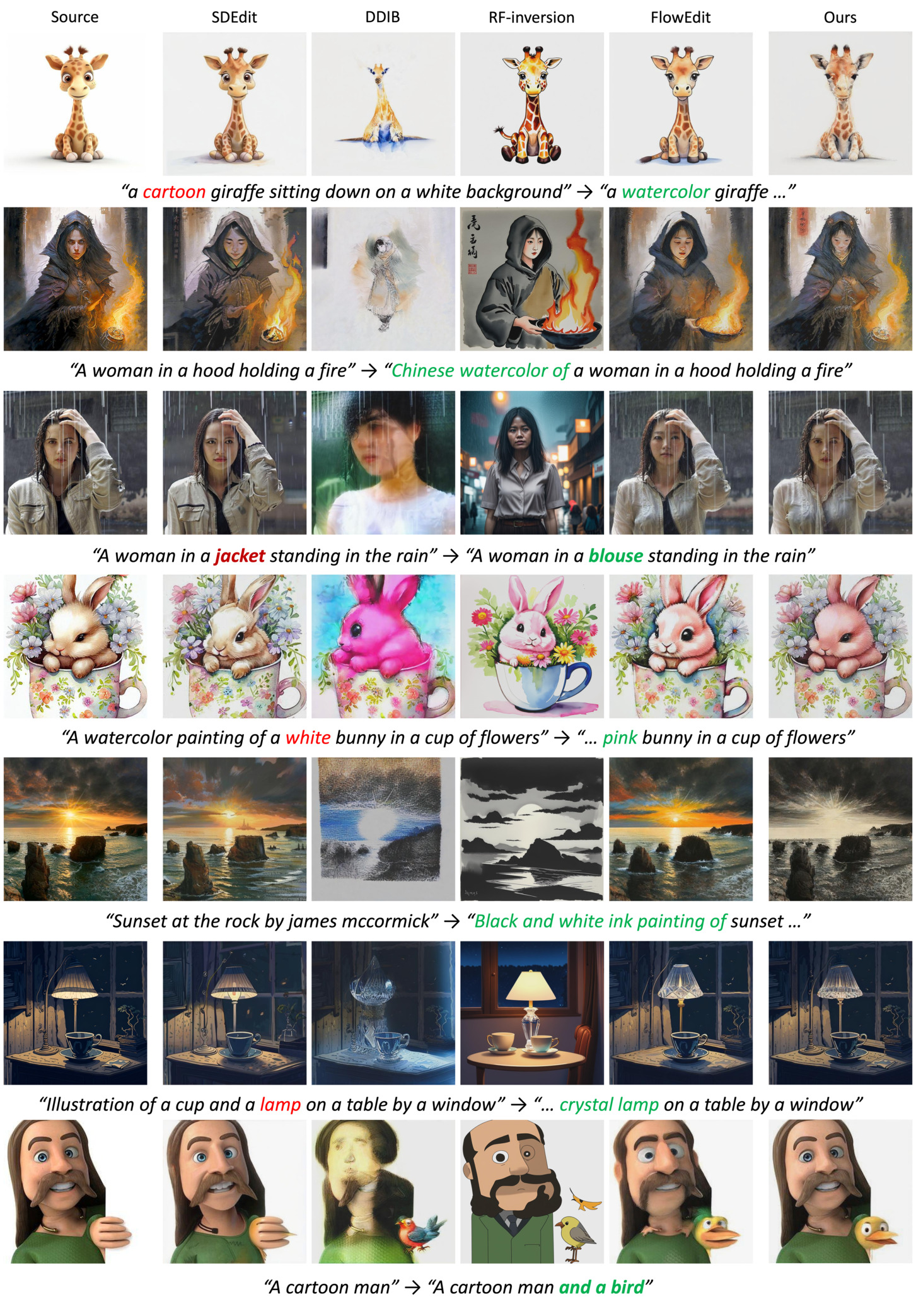}
    \caption{Additional qualitative comparison results.}
    \label{fig:supple_qualitative1}
\end{figure}

\subsection{Ablation}
The proposed method includes two hyper-parameters: $\omega$, which controls efficient CFG, and $\zeta$, which weights the source consistency term derived from the flow-matching regularization.
As discussed in Section~\ref{sec:abl} of the main paper, there is a trade-off between semantic alignment with the target text and structural consistency with the source image. 
Fig.~\ref{fig:abl_zeta} presents qualitative examples from the ablation study, focusing on the effect of $\zeta$ with a fixed $\omega=10.0$.
While the editing results may vary across samples even for the same $\zeta$, we find that $\zeta=0.01$ consistently yields robust and balanced performance, effectively satisfying both the editing instruction and source structure preservation across diverse cases.

\begin{figure}[t]
    \centering
    \includegraphics[width=\linewidth]{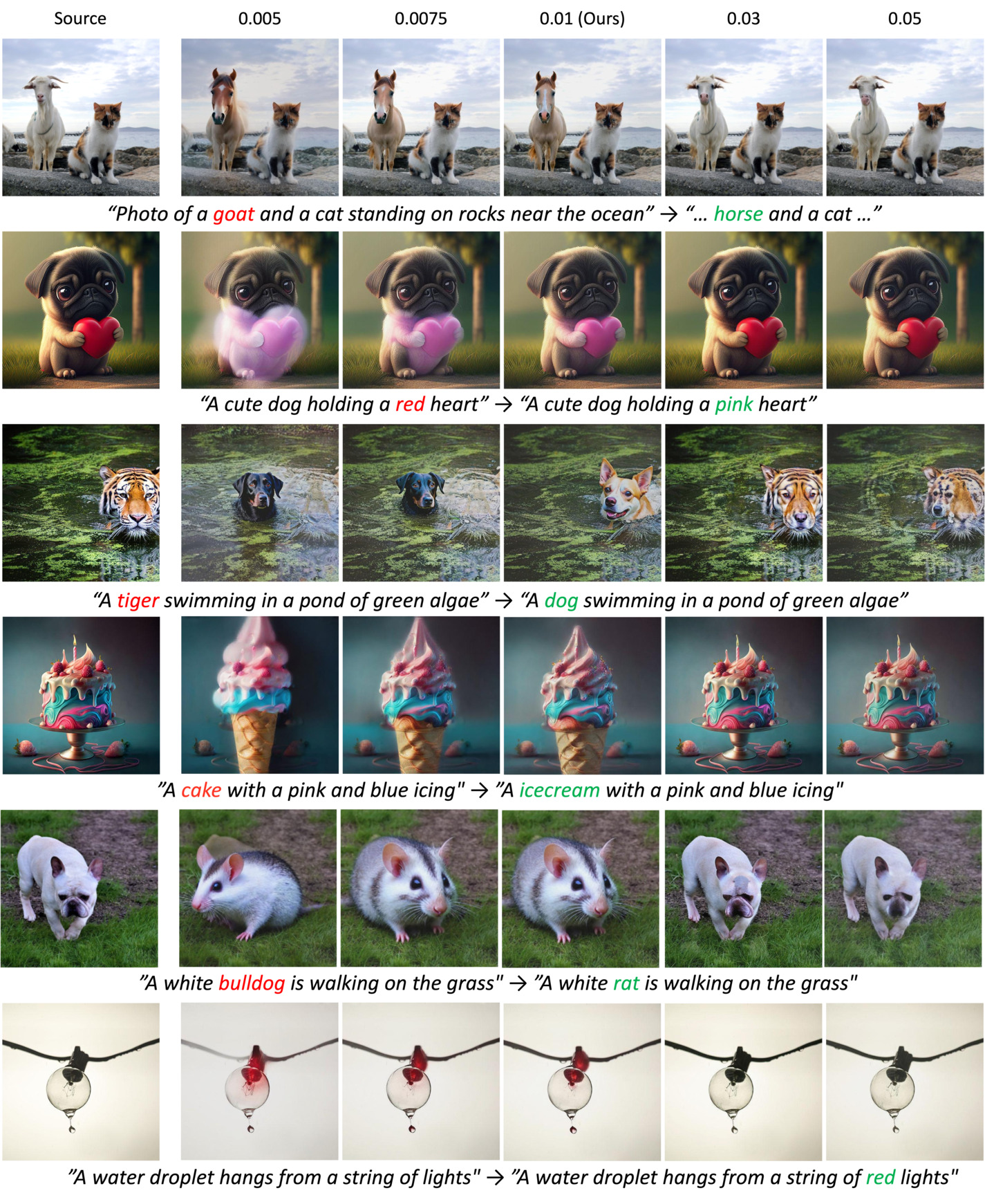}
    \caption{Ablation results for $\zeta$.}
    \label{fig:abl_zeta}
\end{figure}

\section{Additional results}
\label{sec:additional}
As we demonstrated in the last section of the main paper,  FlowAlign can be extended beyond text-based image editing. In this section, we provide details and additional examples for video editing and 3D editing via Gaussian splatting.

\subsection{Further Applications - Video Editing} 
FlowAlign is fundamentally a text-based image editing method built on an image flow-based model. In contrast, video editing typically relies on generative models trained on video datasets with temporal attention mechanisms to enhance temporal consistency~\cite{jeong2024vmc, liu2024video, jeong2024dreammotion, park2025spectral}.
However, video editing can also be viewed as a sequence of image editing tasks. Extending the applicable range of image editing methods to video could offer practical benefits, such as reduced training and inference costs.

While the proposed method does not explicitly enforce temporal consistency, we can still apply it independently to each video frame. For this experiment, we use the DAVIS dataset~\cite{Caelles_arXiv_2019} and extract the source text prompt using LLaVA~\cite{liu2023visual}, conditioned on the middle frame of each video.
For the hyperparameter $\lambda$, we use a constant of 0.01, consistent with the main experiments. For $\omega$, we vary it over [5.0, 7.5, 10.0, 13.5] and qualitatively select the best-performing value.

Figures~\ref{fig:app_video_1} and~\ref{fig:app_video_2} present examples of edited video frames, focusing on texture and object editing tasks.
Due to the strong source structure consistency of the proposed method, the video background remains well-preserved. However, temporal consistency for the edited object is limited, as no explicit constraint is imposed- for example, the head of the swan in Figure~\ref{fig:app_video_1}.
Nonetheless, these results highlight the potential of the proposed method as a lightweight solution for video editing.

\begin{figure}[t]
    \centering
    \includegraphics[width=\linewidth]{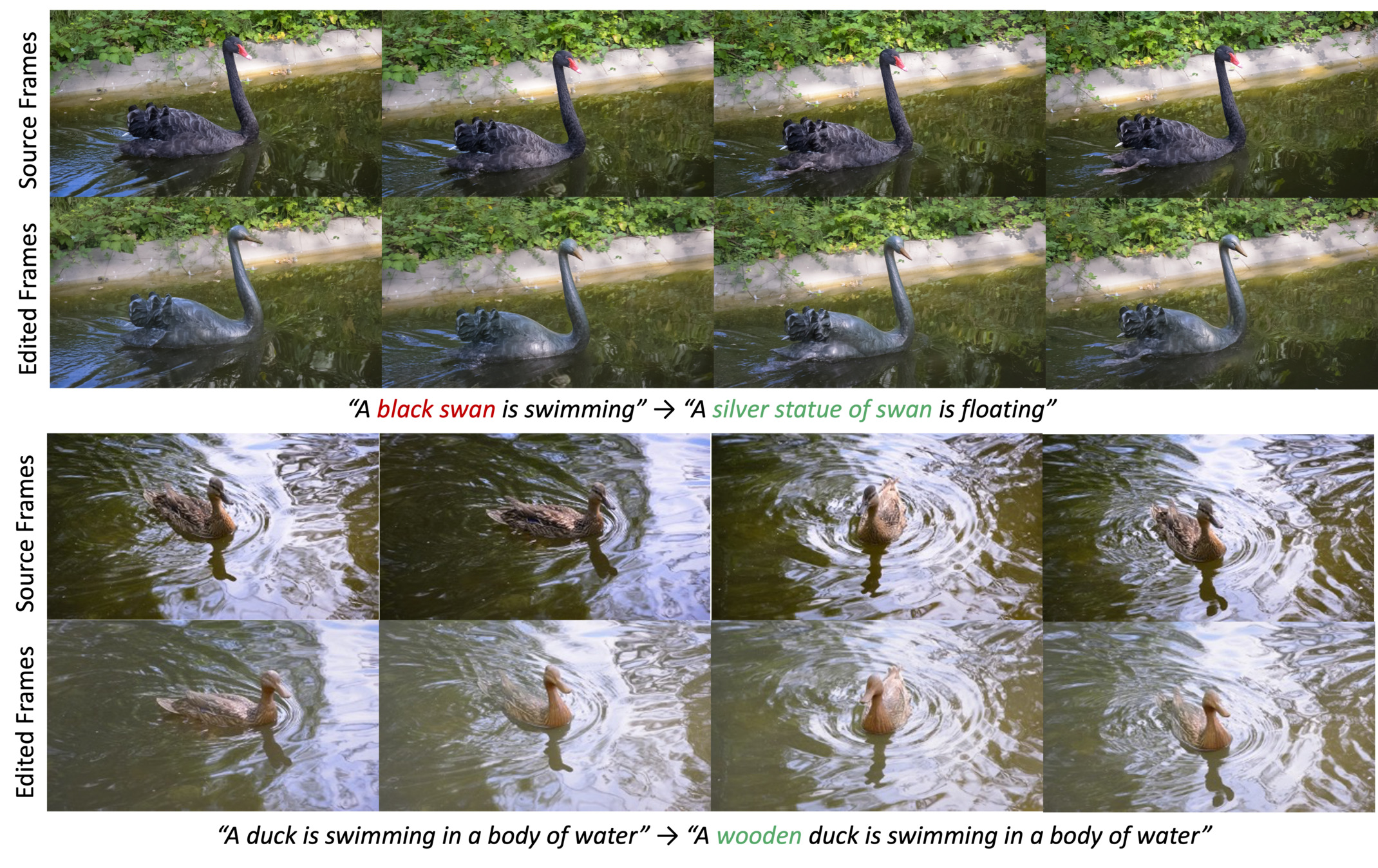}
    \caption{Additional results for video editing: texture change.}
    \label{fig:app_video_1}
\end{figure}

\begin{figure}[t]
    \centering
    \includegraphics[width=\linewidth]{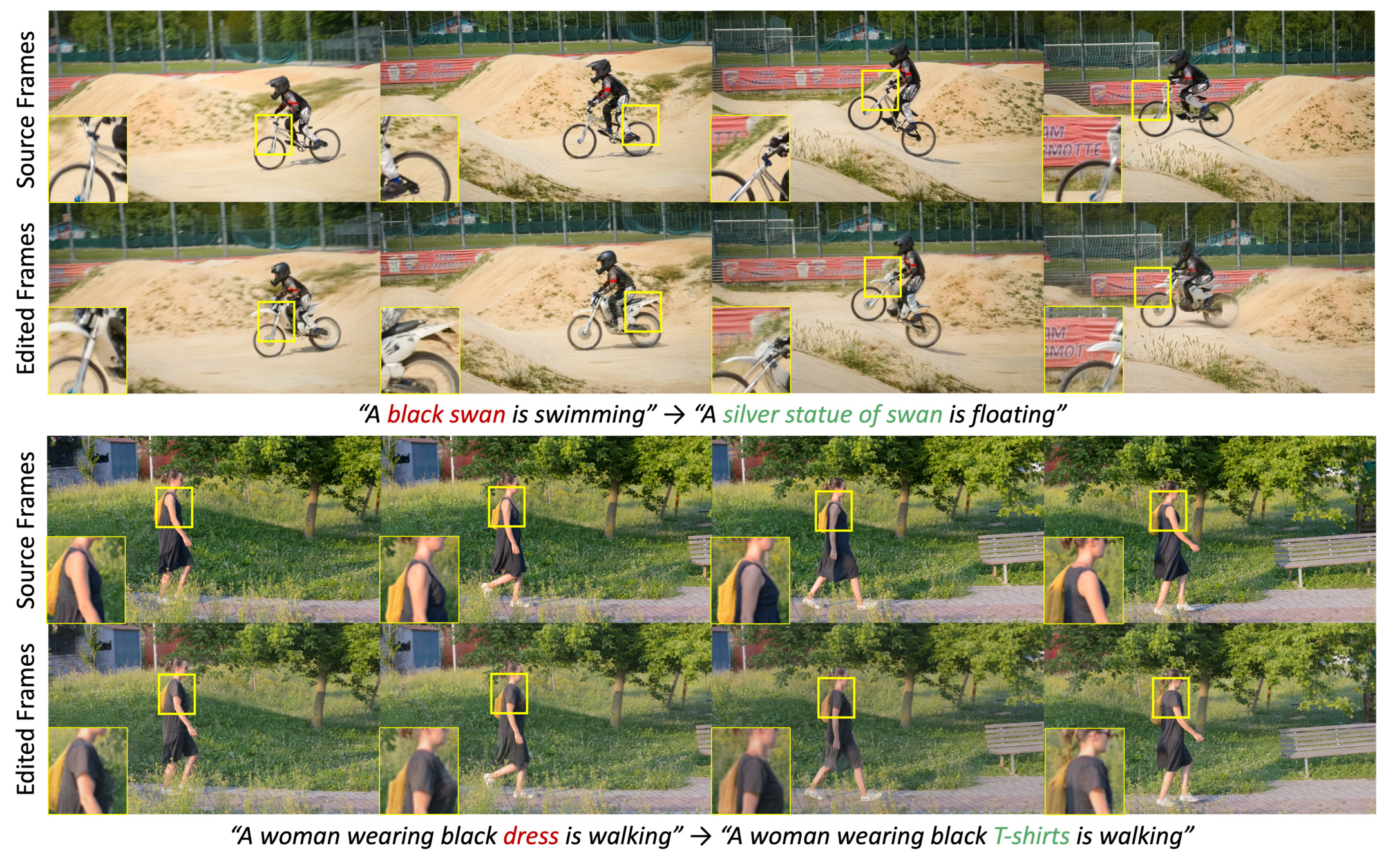}
    \caption{Additional results for video editing: object change.}
    \label{fig:app_video_2}
\end{figure}

\subsection{Further Applications - Gaussian Splatting Editing} 
Recall that one of the main contributions of FlowAlign is introducing a regularization term for inversion-free image editing method, which has a goal of simulating the flow ODE between two image samples.
Accordingly, we can apply FlowAlign for  a 3D editing via Gaussian splatting. The basic idea is to distill the knowledge of text-conditioned image prior of flow model to update Gaussians.

The common way to distill the prior knowledge to parameters of generators, such as NeRF or Gaussian Splatting, is to optimize those parameters with score distillation gradients. 
More generally, it involves guiding the parameters through the use of a pre-trained denoiser acting as a critic.
%
From the result of Proposition 1, we construct a regularized trajectory between source and target images and the drift term \eqref{eqn:final_drift} corresponds to gradient reflecting the editing direction. Specifically, we can define a editing loss function that satisfies
\begin{align}
\nabla_{\vx_t} \mathcal{L}_{FA} := \vv_t(\vp_t, \vc_{tgt}) - \vv_t (\vq_t, \vc_{src}) + \gamma (\mathbb{E}[\vp_0|\vp_t] - \mathbb{E}[\vq_0|\vq_t]),
\end{align}
where we assume that the Jacobians $\frac{\partial \vv_t(\vq_{t})}{\partial \vx_t}$ and $\frac{\partial \vv_t(\vp_{t})}{\partial \vx_t}$ are identity matrices.
Then, by using $\mathcal{L}_{FA}$ as a critic for updating Gaussians, we can guide them toward target Gaussian whose rendered views are high likely sample in perspective of flow model.
The only one we should consider is extending $\mathcal{L}_{FA}$ to contain parameters $\psi$ of differentiable generator $\vg$.

Because we use flow model defined in latent space, we compute the loss function using
$\vq_{t,c}= (1-t)\vx_{src, c} + t\epsilonb$, $\vp_{t,c} = \vq_{t,c} + \vx_{t,c} - \vx_{src, c}$.
Here, $c$ denotes sampled camera view, $\vx_{src, c} = \mathcal{E}(\vg(\psi_{src}, c))$ denotes the latent code of rendered view from initial Gaussian,  $\vx_{t, c} = \mathcal{E}(\vg(\psi, c))$ denotes the latent code of rendered vidw from current Gaussian, $\mathcal{E}$ denotes the pre-trained encoder of a VAE and $\epsilonb \sim \mathcal{N}(0, \boldsymbol{I})$.
 By using the chain rule, we obtain the following gradient for flow-based 3D editing:
\begin{align}
\nabla_{\psi} \mathcal{L}_{FA} &:=\left[  \nabla_{\vx_t}\mathcal{L}_{FA}\right]\frac{\partial \vx_t}{\partial \psi}\\
& = \left[\vv_t(\vp_{t,c}, \vc_{tgt}) - \vv_t(\vq_{t,c}, \vc_{src}) +\gamma 
(\mathbb{E}[\vp_0|\vp_t] - \mathbb{E}[\vq_0|\vq_t])\right] \frac{\partial \vx_t}{\partial \psi},
\label{eqn:gauss_grad}
\end{align}
and use gradient descent to update the $\psi$ as 
\begin{equation}
    \psi = \psi - \eta_t \nabla_\psi \mathcal{L}_{FA}.
\end{equation}
For the semantic guidance term in this gradient, we only apply CFG for $\vp_t$ by setting $c_{src}$ as null-text embedding, as same as the image editing algorithm.

We perform Gaussian Splatting optimization for $3,000$ iterations using a classifier-free guidance (CFG) weight $\omega$ set to either 70 or 100.
%
Experiments are conducted on real-world scenes using datasets from IN2N~\cite{haque2023instruct} and PDS~\cite{koo2024posterior}. For comparison, we benchmark against existing distillation methods, including SDS~\cite{poole2022dreamfusion}, DDS~\cite{hertz2023delta}, and PDS~\cite{koo2024posterior}, as well as the non-distillation baseline IGS2GS~\cite{igs2gs}. For baselines using distillation methods we used Stable Diffusion 3.0 and for IGS2GS~\cite{igs2gs}, we use UltraEdit~\cite{zhao2024ultraedit} in replacement of Instruct-Nerf2Nerf ~\cite{haque2023instruct} for flow-based models.
One can find the pseudocode for the 3D editing with FlowAlign in Algorithm~\ref{alg:method_3d}, and additional results on Figure~\ref{fig:3d_qualitative} and Figure \ref{fig:3d_video}.
\begin{figure}[t]
    \centering

    \includegraphics[width=\linewidth]{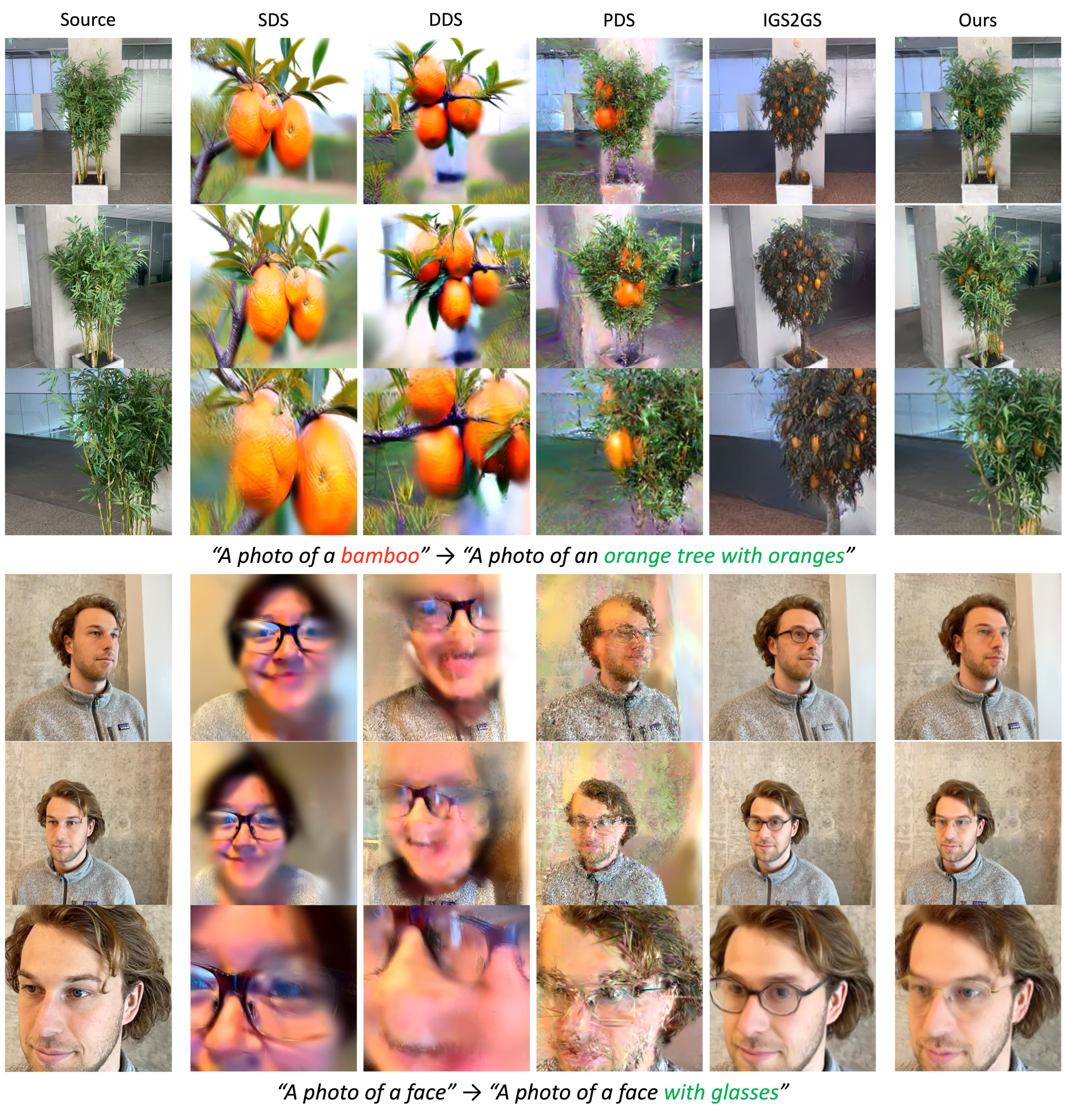}
    \vspace{4pt}
    \caption{Additional results for 3D editing.}
    \label{fig:3d_qualitative}

    \vspace{10pt}

    \newcommand{\columnSpacing}{0.2cm}
    \begin{tabular}{@{}p{0.48\linewidth}@{\hspace{\columnSpacing}}p{0.48\linewidth}@{}}
        \begin{minipage}[t]{\linewidth}
            \begin{tabular}{@{}p{0.48\linewidth}@{\hspace{0.04\linewidth}}p{0.48\linewidth}@{}}
                \animategraphics[loop, width=\linewidth]{6}{videos/bamboo/}{01}{31} &
                \animategraphics[loop, width=\linewidth]{6}{videos/orange_tree/}{01}{31}
            \end{tabular}
        \end{minipage}
        &
        \begin{minipage}[t]{\linewidth}
            \begin{tabular}{@{}p{0.48\linewidth}@{\hspace{0.04\linewidth}}p{0.48\linewidth}@{}}
                \animategraphics[loop, width=\linewidth]{6}{videos/face/}{01}{31} &
                \animategraphics[loop, width=\linewidth]{6}{videos/glasses/}{01}{31} \\
                \vspace{0.1cm} & \vspace{0.1cm} \\
                \animategraphics[loop, width=\linewidth]{6}{videos/yuseung/}{01}{31} &
                \animategraphics[loop, width=\linewidth]{6}{videos/joker/}{01}{31}
            \end{tabular}
        \end{minipage}
    \end{tabular}

    \caption{\footnotesize
        Video results of 3D editing. \textit{Click each image to play the video in Acrobat Reader.}
    }
    \label{fig:3d_video}
\end{figure}


\begin{algorithm}[t]
\caption{Algorithm of FlowAlign for 3D editing}\label{alg:method_3d}
\begin{algorithmic}[1]

\Require Source parameter $\psi_{src}$, Pre-trained flow-based model $\vv_\theta$, VAE encoder $\mathcal{E}$, Source/Target text embeddings $\vc_{src}, \vc_{tgt}$, CFG scale $\omega$, camera views $Cams$, Noise Schedule $\sigma_t$, differentiable generator $\vg$, min and max timestep $T_{min}, T_{max}$ 

\State $\psi \gets \psi_{src}$
\State \textit{// Freeze VAE encoder $\mathcal{E}$ and flow model $\vv_\theta$}
\State $\vx_{src,c} \gets \mathcal{E}(\vg(\psi_{src},c))$
\For{$t: T_{max} \rightarrow T_{min}$}
    \State $\epsilon \sim \mathcal{N}(0, \sigma^2 \rmI) \quad c\sim \mathcal{U}(Cams)$
    \State $\vx_{t,c} \gets \mathcal{E}(\vg(\psi,c))$
    \State $\vq_{t,c} \gets (1-\sigma_t) \vx_{src,c} + \sigma_t \epsilon$
    \State $\vp_{t,c} \gets \vx_{t,c} - \vx_{src,c} + \vq_{t,c}$
    \State $\vv^\theta(\vp_{t,c}) = \vv^\theta(\vp_{t,c}, \vc_{src}) + \omega\left[\vv^\theta(\vp_{t,c}, \vc_{tgt}) - \vv^\theta(\vp_{t,c}, \vc_{src})\right], \quad\vv^\theta(\vq_{t,c}) = \vv^\theta(\vq_{t,c}, \vc_{src})$
    \State $\mathbb{E}[\vp_0|\vp_t] \gets \vp_{t,c} - t \vv^\theta(\vp_{t,c}), \quad \mathbb{E}[\vq_0|\vq_t] \gets \vq_{t,c} - t \vv^\theta(\vq_{t,c})$
    \State $\nabla_\psi \mathcal{L}_{FA} \gets [(\vv_\theta(\vp_{t,c}) - \vv_\theta(\vq_{t,c}) + 
    \gamma
    (\mathbb{E}[\vp_0|\vp_t]-\mathbb{E}[\vq_0|\vq_t)){\frac{\partial \vx}{\partial \psi}}]$
    \State $\psi \gets \psi - \eta_t \nabla_\psi \mathcal{L}_{FA}$
\EndFor
\State \textbf{return} $\psi$
\end{algorithmic}
\end{algorithm}

\begin{algorithm}[!t]
\caption{Algorithm of FlowAlign with FLUX (Guidance distilled model)}\label{alg:method_flux}
\begin{algorithmic}[1]
\Require Source image $\vx_{src}$, Pre-trained flow model $\vv^\theta$, VAE encoder and Decoder $\mathcal{E}, \mathcal{D}$, 
Source/Target text embeddings $c_{src}, c_{tgt}$, CFG scales $\omega_{src}, \omega_{tgt}$, source consistency scale $\zeta$
\State $\vx_t \gets \mathcal{E}(\vz_{src})$
\For{$t: 1\rightarrow 0$}
    \State $\epsilon \sim \mathcal{N}(0, \rmI)$
    \State $\vq_t \gets (1-t) \vx_{src} + t \epsilonb$
    \State $\vp_t \gets \vx_t - \vx_{src} + \vq_t$
    \State $\vv^\theta(\vp_t) := \vv^\theta(\vp_t, \vc_{tgt}, \omega_{tgt}), \quad \vv^\theta(\vq_t) := \vv^\theta(\vq_t, \vc_{src}, \omega_{src})$
    \State $\mathbb{E}[\vp_0|\vp_t] \gets \vp_t - t \vv^\theta(\vp_t), \quad \mathbb{E}[\vq_0|\vq_t] \gets \vq_t - t \vv^\theta(\vq_t)$
    \State $\vx_t \gets \vx_t + \left[\vv^\theta(\vp_t) - \vv^\theta(\vq_t)\right] dt + \zeta (\mathbb{E}[\vq_0|\vq_t]-\mathbb{E}[\vp_0|\vp_t])$
\EndFor
\State $\vz_{edit} \gets \mathcal{D}(\vx_t)$
\end{algorithmic}
\end{algorithm}

\subsection{Further Variants - FLUX}
\begin{figure}[!t]
    \centering
    \includegraphics[width=\linewidth]{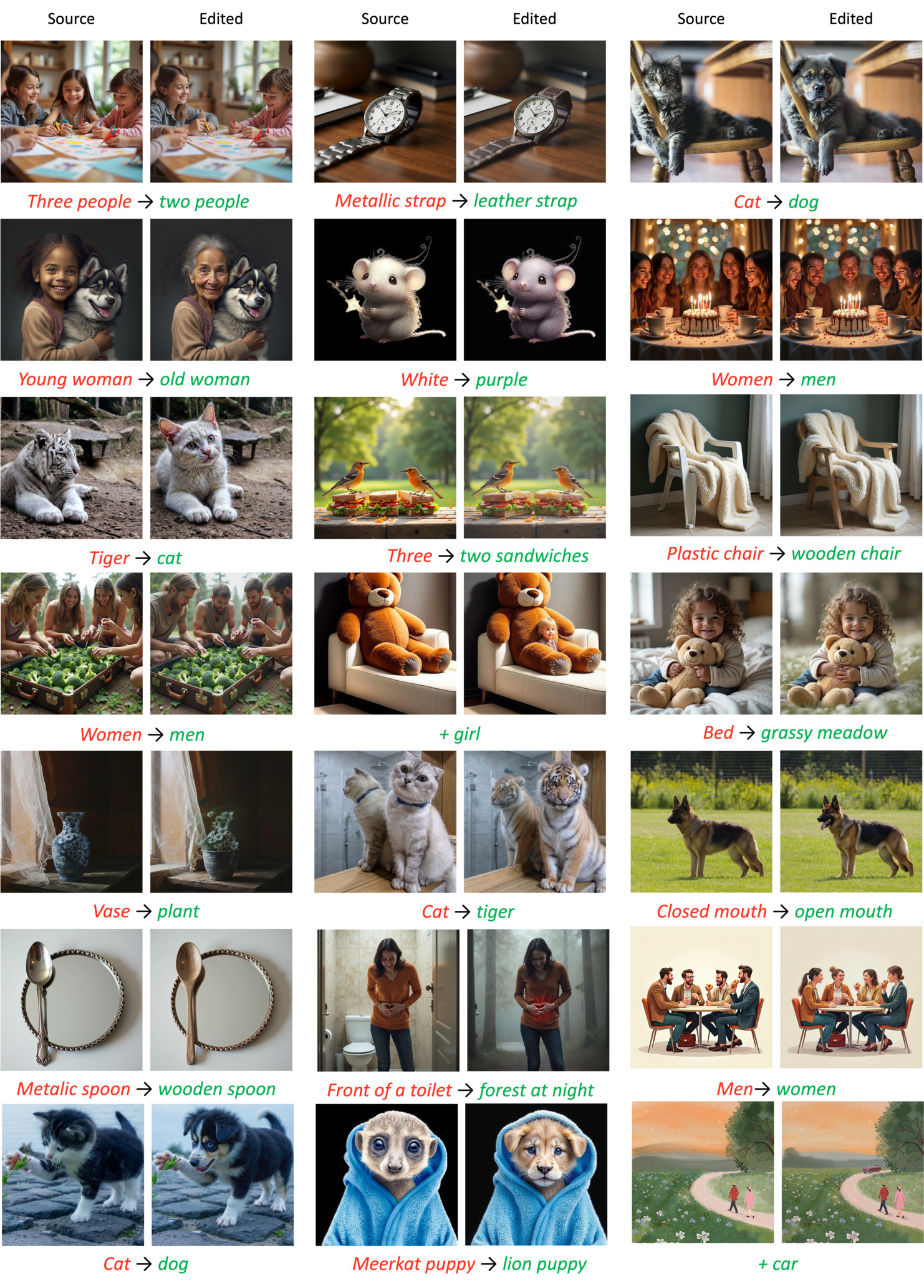}
    \vspace{-0.6cm}
    \caption{Qualitative image editing results using FLUX.}
        \vspace{-0.6cm}
    \label{fig:flux}
\end{figure}
Our main experiments are conducted using by leveraging pre-trained flow model, Stable Diffusion 3.0. Since the proposed method regulates the sampling ODE via a flow matching cost, it is also compatible with other flow-based models, such as FLUX.
Following the setup in FlowEdit~\cite{kulikov2024flowedit}, we additionally evaluate our method using FLUX as the backbone model. We incorporate the source consistency term derived from optimal control with flow matching regularization, as described in Algorithm~\ref{alg:method_flux}.
As FLUX.1-dev is a guidance-distilled model that directly takes the CFG scale as input, there is no need to apply the modified CFG strategy used in Stable Diffusion 3.0 case. Therefore, we simply provide the CFG scale to both $\vv(\vp_t)$ and $\vv(\vq_t)$, following the approach used in FlowEdit.
Figure~\ref{fig:flux} illustrates edited results generated by the proposed method implemented with FLUX.1-dev. Similar to the results with Stable Diffusion 3.0, the outputs effectively reflect the intended editing direction specified by the text prompts while preserving source structures. 
These results imply that the proposed method is broadly applicable to flow-based models for image editing.

\end{document}